\documentclass{article} 
\PassOptionsToPackage{authoryear,round}{natbib}
\usepackage[preprint]{neurips_2026}
\makeatletter
\renewcommand{\@noticestring}{}
\makeatother

\usepackage{amsmath,amsfonts,bm}

\def\eqref#1{equation~\ref{#1}}

\def\1{\bm{1}}

\DeclareMathAlphabet{\mathsfit}{\encodingdefault}{\sfdefault}{m}{sl}
\SetMathAlphabet{\mathsfit}{bold}{\encodingdefault}{\sfdefault}{bx}{n}

\usepackage[utf8]{inputenc} 
\usepackage[T1]{fontenc}    
\usepackage{hyperref}       
\hypersetup{hidelinks,pdftitle={When to Commit and When to Defer: Maturing Markov Decision Processes under Refining Information and Expiring Opportunities},pdfauthor={Jiaxi Liu, Aiping Yang, Yuhang Yang, Shuqi Zhang, Zewei Dong, Jiangming Yang, Xuebin Chen}}
\usepackage{url}            
\usepackage{booktabs}       
\usepackage{amsfonts}       
\usepackage{nicefrac}       
\usepackage{microtype}      
\usepackage{xcolor}         
\usepackage{graphicx}
\usepackage{subcaption}
\usepackage{amsmath}
\usepackage{amsthm}
\usepackage{tabularx}
\usepackage{amssymb}

\makeatletter
\renewcommand{\eqref}[1]{\textup{\tagform@{\ref{#1}}}}
\makeatother

\newcolumntype{Y}{>{\raggedright\arraybackslash}X}

\theoremstyle{definition}
\newtheorem{definition}{Definition}[section]

\theoremstyle{plain}
\newtheorem{theorem}{Theorem}[section]

\newtheorem{remark}{Remark}[section]

\title{When to Commit and When to Defer: Maturing Markov Decision Processes under Refining Information and Expiring Opportunities}

\author{%
  Jiaxi~Liu$^{*\dagger}$ \\
  Ant International\\
  School of Economics, Sichuan University\\
  \texttt{liujiaxi@stu.scu.edu.cn}
  \And
  Aiping~Yang$^{*}$ \\
  Ant International\\
  \texttt{yangaiping.yap@ant-intl.com}
  \And
  Yuhang~Yang$^{*}$ \\
  Ant International\\
  \texttt{yumen.yyh@ant-intl.com}
  \And
  Shuqi~Zhang$^{*}$ \\
  Ant International\\
  \texttt{zhangshuqi.zsq@ant-intl.com}
  \And
  Zewei~Dong$^{\dagger}$ \\
  Ant International\\
  \texttt{zewei.dong@ant-intl.com}
  \And
  Jiangming~Yang \\
  Ant International\\
  \texttt{jiangming.yang@ant-intl.com}
  \And
  Xuebin~Chen \\
  School of Economics, Sichuan University\\
  School of Economics, Fudan University\\
  \texttt{chenxb@fudan.edu.cn}
}

\begin{document}

\maketitle

\begin{center}
\footnotesize $^{*}$Equal contribution. $^{\dagger}$Corresponding authors: Jiaxi Liu and Zewei Dong.
\end{center}

\begin{abstract}
Sequential decisions often become better informed as action opportunities expire. We introduce \emph{Maturing Markov Decision Processes} (MMDPs), a structured subclass of augmented finite-horizon Markov decision processes, for organizing when to commit and when to defer. The Expiring-Actions-First Principle provides a sufficient condition for deferral based on information gain and delay loss. When this condition cannot be certified, model bounds guide which persistent actions remain eligible. This structure motivates stage-local policies, action abstraction, and policy-guided planning. Replenishment experiments show improved learning and test performance under the MMDP interface and, in the primary setting, a larger test-performance gain from forecast refinement. With costly deferral, the onset of MMDP's advantage over fixed deferral is consistent with the theoretical analysis, illustrating how the extended Expiring-Actions-First Principle can guide which persistent actions to consider early. In the synthetic cash-management task, MMDP improves PPO's test returns across both network sizes, and this advantage persists with action abstraction and policy-guided search in the larger network. A production-scale cash-management case study adds complementary evidence. In maze navigation, MMDP yields larger gains in the more complex of the two settings. Taken together, these results show that the MMDP formulation can improve learning and planning in settings with refining information and expiring opportunities.
\end{abstract}

\section{Introduction}
\begin{figure}[h]
    \centering
    \includegraphics[width=0.75\linewidth]{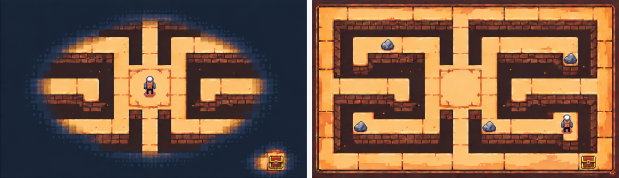}
    \caption{\textbf{Increasing information, expiring opportunities.} Early decisions use coarse information but greater flexibility; later decisions use richer information but fewer actions.}
    \label{fig:intro_asymmetry}
\end{figure}
Many real-world sequential decision problems evolve in a structurally asymmetric way: as a decision cycle unfolds, the agent receives increasingly informative state signals through forecasts, observations, or operational updates, while the feasible action set contracts due to expiring opportunities, operational cutoffs, irreversible commitments, or resource constraints (see Figure~\ref{fig:intro_asymmetry}). Such dynamics arise across domains including dynamic pricing \citep{netessine2006dynamic,cheung2017dynamic}, inventory control \citep{wang2008inventory,liu2023ai,xie2026deepstock}, resource allocation \citep{ashlagi2018maximum,jacquillat2024learning}, and financial liquidity management \citep{pate1990dynamic,castro2025estimating}. Early decisions therefore operate under coarse information but high flexibility, whereas later decisions benefit from richer information but reduced optionality, creating a commitment--deferral trade-off.

A finite-horizon Markov decision process (MDP) can encode stage indices, forecast updates, and feasibility masks. This representation alone, however, does not organize decisions around which opportunities must be resolved before they expire and which choices can remain open for better information. We introduce \emph{Maturing Markov Decision Processes} (MMDPs) to make this distinction explicit: \emph{maturing} describes a process that becomes more informed while intervention opportunities progressively expire.

Our contributions are fourfold. First, we define MMDPs through strictly refining information and expiring opportunities, distinguishing expiring actions from persistent ones. Second, we use this distinction to establish the Expiring-Actions-First Principle by comparing information gain with delay loss. When deferral is costly, we characterize which persistent actions can improve on the best expiring choice. Third, we develop structure-aware reinforcement learning for MMDPs with stage-local policies, action abstraction, and policy-guided planning. Fourth, we evaluate the benefits of using MMDPs to organize decisions around deferral and information updates in replenishment, cash management, and maze navigation. A production-scale case study adds simulator, backtest, and deployment evidence.

\section{Related Work}

\textbf{Standard and constrained MDPs.} Classical and constrained MDPs model sequential control and feasibility requirements \citep{puterman2014markov,sutton1998reinforcement,altman2021constrained}. An augmented finite-horizon MDP can encode stage indices, observations, and action masks. MMDPs do not extend this representational generality; they isolate the coupled refinement of information and expiration of action opportunities.

\textbf{Partial observability.} POMDPs address hidden-state inference and belief-based control \citep{kaelbling1998planning,smallwood1973optimal,spaan2012partially}, while recurrent RL offers a practical approach to history-dependent decisions \citep{ni2021recurrent}. MMDPs emphasize a complementary structure: how information revelation interacts with expiring opportunities, given an information state sufficient for control.

\textbf{Information updates and decision timing.} Procurement models analyze ordering flexibility under forecast updates \citep{donohue2000efficient,sethi2004quantity}. Related timing constraints appear in pricing with limited price changes \citep{cheung2017dynamic}, online matching with departure deadlines \citep{ashlagi2018maximum}, and intraday payment decisions \citep{castro2025estimating}. These problems differ in how information arrives and opportunities disappear. MMDPs focus on their joint stage-level structure, establish a sufficient condition for action restriction, and evaluate that restriction as an interface for learning and planning.

\textbf{Structured actions and search.} Structured and parameterized-action RL exploit combinatorial or hybrid decision representations \citep{khalil2017learning,nazari2018reinforcement,delarue2020reinforcement,hausknecht2015deep,masson2016reinforcement}. MMDPs add a temporal distinction between expiring and persistent decisions. Policy-guided planning also relates to search-based policy improvement \citep{silver2017mastering,anthony2017thinking}; our controlled evaluation studies inference-time Search with frozen policies.

\section{Maturing Markov Decision Processes}
\label{sec:mmdp}

Maturing Markov Decision Processes (MMDPs) are a structured subclass of augmented finite-horizon MDPs in which information becomes richer while semantic action opportunities permanently expire.

\subsection{Information--Action Asymmetry}
\label{sec:mmdp:properties}

\begin{definition}[Maturing Markov Decision Process]
\label{def:mmdp}
An MMDP is a finite-horizon Markov decision process with effective stages $n=0,\ldots,N$, specified by
\[
\mathcal M^{\mathrm M}=\bigl(\mathcal S^\star,\mathcal A,\{\mathcal S_n,\phi_n,\mathcal A_n,P_n,r_n\}_{n=0}^{N},\Gamma,\gamma\bigr),
\]
whose information representations and structural action sets satisfy
\begin{equation}
\mathcal S_0\prec\mathcal S_1\prec\cdots\prec\mathcal S_N,
\qquad
\mathcal A_0\supsetneq\mathcal A_1\supsetneq\cdots\supsetneq\mathcal A_N.
\label{eq:maturation_structure}
\end{equation}
\end{definition}

Here $\phi_n:\mathcal S^\star\to\mathcal S_n$ maps a full-information description to the stage-$n$ information state, and $\mathcal A_n\subseteq\mathcal A$ records unexpired opportunities in a common semantic action space. The information state together with $n$ is sufficient for Markov control, including any required observable history and remaining horizon. Rewards $r_n(s,a)$ are defined on executable state--action pairs, and $\gamma\in(0,1]$ discounts each underlying MDP transition. Together, $P_n$ and $\Gamma$ specify one transition kernel: $P_n$ governs within-stage transitions, while $\Gamma$ specifies boundary crossing and its successor law.

\paragraph{Strict information refinement.}
The relation $\mathcal S_n\prec\mathcal S_{n+1}$ means that the earlier representation $\phi_n(x)$ can be recovered measurably from $\phi_{n+1}(x)$, but not conversely on the support of the underlying descriptions. Later stages therefore retain the information available earlier while revealing additional decision-relevant distinctions.

\paragraph{Semantic action expiration.}
The relation $\mathcal A_n\supsetneq\mathcal A_{n+1}$ means that some opportunities permanently expire while every later opportunity is already represented earlier with the same semantic meaning. Physically later operations may be represented through valid reservations or registrations, but expired opportunities cannot reopen within the cycle. These structural sets describe stage-level expiration; resource-dependent changes in current executability are represented through the state.

\begin{remark}[Non-Degenerate Stage Partition]
\label{rem:nondegenerate_stage}
A transition from stage $n$ to stage $n+1$ defines a new effective stage only when
$\mathcal S_n\prec\mathcal S_{n+1}
$ and $
\mathcal A_n\supsetneq\mathcal A_{n+1}.$
\end{remark}
The strict action contraction must include the expiration of at least one domain-action opportunity, rather than only a stage-termination control. If either condition fails, the two stages are treated as a single effective stage, with the intervening events retained in the underlying dynamics. Appendix~\ref{sec:nondegenerate_stage} provides further interpretation.

\paragraph{Commitment--Deferral Trade-Off.}
Acting earlier preserves more opportunities but commits under coarser information; waiting permits better-informed choice while forfeiting expiring opportunities. The resulting commitment--deferral trade-off implies neither that every expiring action should be executed nor that waiting is always optimal.

\subsection{Expiring-Actions-First Principle}
\label{sec:mmdp:priority}

At a nonfinal stage $n$, let $\mathcal D_n=\mathcal A_n\setminus\mathcal A_{n+1}$ and $\mathcal S_{n+1}=\mathcal S_n\times\mathcal Z_{n+1}$, where $\mathcal Z_{n+1}$ is the space of newly revealed information. Taking the executable stage-termination action $a_n^{\mathrm{end}}\in\mathcal D_n$ reveals $z\in\mathcal Z_{n+1}$ and yields the successor state $s'=(s,z)$ in one transition, with $s'\sim\Gamma(\cdot\mid s)$. By Remark~\ref{rem:nondegenerate_stage}, $\mathcal D_n$ contains an expiring domain action, and the signal $z$ provides a strict refinement of the stage-$n$ information. Since feasibility can also depend on the state, the maximizations for the current decision use only actions executable at $s$.

Over persistent actions executable at $s$, compare the best execution now ($J_{\mathrm{now}}$), selection now for execution after the boundary ($J_{\mathrm{delay}}$), and selection and execution after observing $z$ ($J_{\mathrm{adapt}}$). All three values are discounted to $s$ and allow optimal continuation in the full model.

\paragraph{Information Gain.}
At state $s$, the information revealed at the next stage can improve the choice of a persistent action. We define this expected gain by comparing selection after observing $z$ with selection before observing $z$, with execution delayed in both cases:
\begin{equation}
\delta_n(s)
:=
J_{\mathrm{adapt}}-J_{\mathrm{delay}}
\geq0.
\label{eq:boundary_voi}
\end{equation}

\paragraph{Delay Cost.}
Let $c_n(s)\geq0$ uniformly bound the loss from delaying any fixed persistent action, with optimal continuation and all discounted returns measured from $s$. Then
\begin{equation}
J_{\mathrm{now}}-J_{\mathrm{delay}}
\leq c_n(s).
\label{eq:fixed_action_delay_bound}
\end{equation}

These quantities give $J_{\mathrm{adapt}}-J_{\mathrm{now}}\geq\delta_n(s)-c_n(s)$.

\begin{theorem}[Expiring-Actions-First Principle]
\label{thm:expiring_priority}
Under the preceding conditions, if $c_n(s)\leq\delta_n(s)$, an optimal current decision can be selected as
\begin{equation}
\pi_n^\star(s)\in\arg\max_{a\in\mathcal D_n}Q_n^\star(s,a).
\label{eq:expiring_action_reduction}
\end{equation}
\end{theorem}
The theorem reduces the set needed to find an optimal current decision to $\mathcal D_n$. When $c_n(s)>\delta_n(s)$, the stage-wide test leaves deferral unresolved. We then compare each persistent action with the best expiring action, including optimal continuation in both values.

\begin{theorem}[Expiring-Actions-First with Costly Deferral]
\label{thm:action-wise-deferral}
Under the preceding conditions, if $c_n(s)>\delta_n(s)$, an optimal current decision can be selected as
\[
\pi_n^\star(s)\in\arg\max_{a\in\mathcal E_n(s)}Q_n^\star(s,a),
\]
where
\begin{equation}
\mathcal E_n(s)
=\mathcal D_n\cup\left\{a\in\mathcal A_{n+1}:
Q_n^\star(s,a)>\max_{d\in\mathcal D_n}Q_n^\star(s,d)\right\}.
\label{eq:action-wise-deferral}
\end{equation}
\end{theorem}

Appendix~\ref{app:expiring_priority_proof} proves both theorems.

\section{Structure-Aware Reinforcement Learning for MMDPs}
\label{sec:structure-aware_reinforcement_learning}

\subsection{Stage-Local Policies}
\label{sec:semantic-interface}

We use the MMDP partition into expiring and persistent actions to construct the policy's candidate set. If the condition in Theorem~\ref{thm:expiring_priority} is certified, $\mathcal D_n$ suffices. Otherwise we use a set $\widehat{\mathcal E}_n(s)$ constructed from model bounds to retain the actions in the exact set $\mathcal E_n(s)$ of Theorem~\ref{thm:action-wise-deferral}. Restricting both sets to executable actions \citep{huang2022invalidmasking}, the nonfinal-stage policy satisfies
\begin{equation}
\operatorname{supp}\pi_\theta(\cdot\mid s,n)\subseteq
\begin{cases}
\mathcal D_n, & \text{if } c_n(s)\leq\delta_n(s)\text{ is certified},\\[2pt]
\widehat{\mathcal E}_n(s), & \text{otherwise}.
\end{cases}
\label{eq:stage-local-policy}
\end{equation}
To construct $\widehat{\mathcal E}_n(s)$, we retain all expiring actions and exclude a persistent action if its upper value bound is at most the value of a feasible plan starting with an expiring action. A feasible plan starting with a persistent action that exceeds an upper bound on the best expiring value certifies that every optimal current decision is persistent (Appendix~\ref{app:semantic-interface-details}).

We call a policy that follows Eq.~\eqref{eq:stage-local-policy} an \emph{MMDP policy}. When deferral is known to be costless ($c_n(s)=0$), the first branch applies under the assumptions of Theorem~\ref{thm:expiring_priority}. A policy that always uses $\mathcal D_n$ at nonfinal stages without checking the deferral condition is \emph{expiring-only} (EO). Both branches retain the stage-termination action $a_n^{\mathrm{end}}\in\mathcal D_n$. The final stage uses executable actions in $\mathcal A_N$. The \emph{Flat} baseline allows the controller to choose any action in $\mathcal A_n$ that is executable at $s$.

\subsection{Action Abstraction and State-Dependent Execution}
\label{sec:action-execution}

The stage-local action set can still contain continuous or hybrid actions. Action abstraction (AA) further simplifies this stage-local action space: an abstract action $\tilde a$ specifies \emph{which} eligible decision to resolve, while a state-dependent executor $E_n(s,\tilde a)\in\mathcal A_n$ supplies its parameters.

In cash management, the policy selects a transfer edge $(u,v)$ or the stage-termination action, and the executor supplies the amount ratio $\rho$ in $a=(u,v,\rho)$. Flat-AA and MMDP-AA share this execution rule under their respective interfaces. Because the executor may omit an optimal concrete action, AA can introduce additional approximation error. Execution details appear in Appendix~\ref{app:action-execution-details}.

\subsection{Policy-Guided Search}
\label{sec:policy-guided-search}

The stage-local interface also provides a search space for refining the learned policy's current choice. The policy guides candidate construction and evaluation order, while an action-value estimate compares their long-term consequences. For the abstract interface, search selects
\begin{equation}
\tilde a^{\mathrm{search}}\in\arg\max_{\tilde a\in\mathcal C_B(s,n;\theta)}\widehat Q(s,n,\tilde a),
\label{eq:bounded-refinement}
\end{equation}
where $\mathcal C_B$ is the nonempty set of candidates scored within a planning budget $B$. Here $\widehat Q$ estimates the return from executing $E_n(s,\tilde a)$ and continuing with the frozen policy. 

In the controlled benchmarks, Search scores candidates using only current observations and public model information, executes the selected action, and replans after the next observation. Implementation and budget details appear in Appendix~\ref{app:policy-guided-search-details}.

\section{Experiments}
\label{sec:experiments}
We test whether organizing commitments and deferrals improves learning and planning as information refines and opportunities expire. Replenishment examines forecast updates, daily review, and costly deferral. Synthetic cash management examines route selection, action abstraction, and bounded planning. Figure~\ref{fig:mmdp_examples} shows their decision cycles. The production-scale case adds simulator, backtest, and deployment evidence. Maze navigation examines how movement and new information shape the opportunities an agent can pursue.

\begin{figure}[htbp]
\centering
\begin{subfigure}[t]{0.48\textwidth}
\centering
\includegraphics[width=\linewidth]{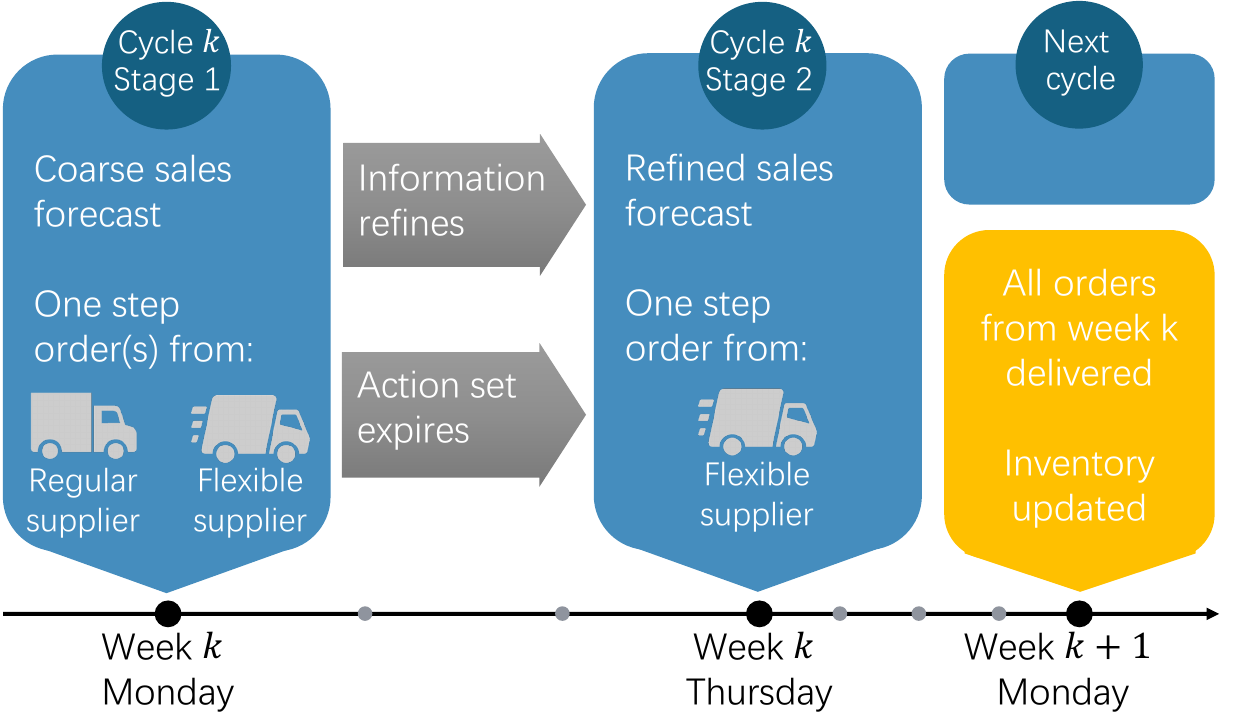}
\caption{Replenishment problem.}
\label{fig:replenishment_mmdp}
\end{subfigure}
\hfill
\begin{subfigure}[t]{0.48\textwidth}
\centering
\includegraphics[width=\linewidth]{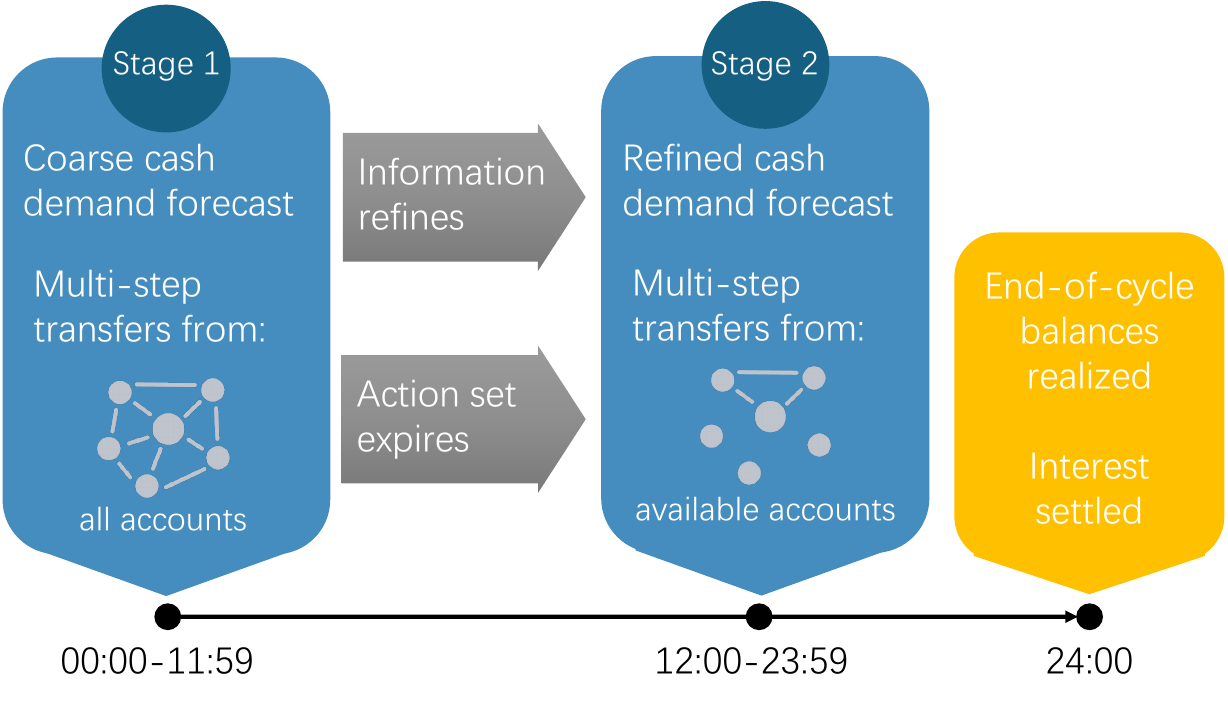}
\caption{Synthetic cash management problem.}
\label{fig:cash_mmdp}
\end{subfigure}
\caption{Information refinement and opportunity expiration within one decision cycle. \textbf{(a)} Replenishment: the regular-supplier ordering window closes before the forecast update, while flexible ordering remains available. Gray dots mark the additional daily reviews. \textbf{(b)} Synthetic cash management: forecasts update as transfer routes expire.}
\label{fig:mmdp_examples}
\end{figure}

\subsection{Staged Replenishment}

We study staged replenishment, motivated by recent RL applications in inventory control \citep{liu2023ai,xie2026deepstock}, dual-sourcing models \citep{veeraraghavan2008now,yee2024dual}, and procurement with forecast updates \citep{donohue2000efficient,sethi2004quantity,federgruen2026sourcing}. Orders are reviewed on Monday and Thursday. The cheaper regular supplier closes before the forecast update, while the more expensive flexible supplier remains available. Early and late orders arrive together the following Monday. Flexible orders have the same unit price at both reviews \citep{milner2005order} and draw from a single capacity budget \citep{feng2010procurement}. This gives $c_n(s)=0$ for deferring a fixed feasible flexible order. The surcharge experiment changes this property.

Flat permits early orders from either supplier; MMDP defers flexible orders. Both use PPO \citep{schulman2017proximal} with identical observation fields, including remaining flexible capacity, and matched networks, action parameterization, and training budgets. All policies are trained from scratch. A control retains the coarse forecast at the later stage. Environment, training, and evaluation protocols appear in Appendix~\ref{app:replenishment-environment-protocol}.

\paragraph{Learning and generalization.}
We compare two-stage review with seven daily ordering reviews \citep{tagaras2001periodic}, where forecasts refine as daily quotas expire. MMDP leads Flat throughout the measured two-stage learning curve, and the early validation gap is wider under daily review (Figure~\ref{fig:learning_curves} and \ref{fig:rep_daily_learning}). MMDP's paired test gain is $132.5$ ($95\%$ CI $[56.6,213.0]$) with two stages and $214.6$ ($[95.8,350.6]$) with daily review at the respective training budgets (Figure~\ref{fig:rep_test} and \ref{fig:rep_daily_test}). Appendix~\ref{app:rep-multi-stage} details the two-stage and daily-review comparison.

\begin{figure}[!t]
\centering
\begin{minipage}{\linewidth}
\captionsetup{type=figure}
\centering
\begin{subfigure}[t]{0.24\linewidth}
    \centering
    \includegraphics[width=\linewidth]{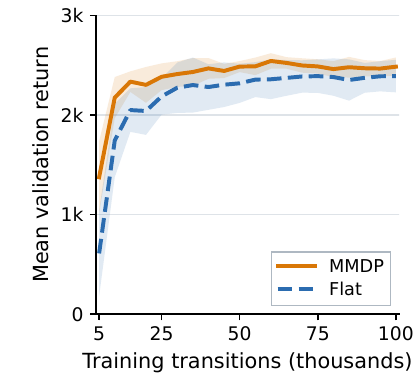}
    \caption{Two-stage learning}
    \label{fig:learning_curves}
\end{subfigure}
\hfill
\begin{subfigure}[t]{0.24\linewidth}
    \centering
    \includegraphics[width=\linewidth]{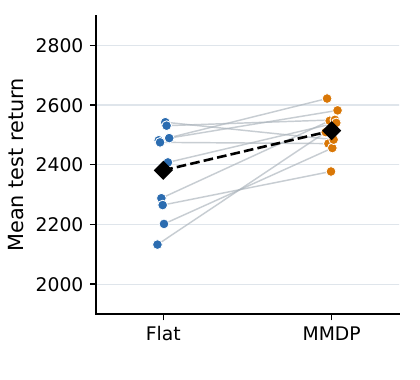}
    \caption{Two-stage test}
    \label{fig:rep_test}
\end{subfigure}
\hfill
\begin{subfigure}[t]{0.24\linewidth}
    \centering
    \includegraphics[width=\linewidth]{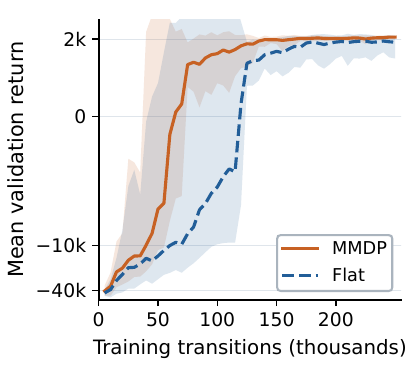}
    \caption{Daily-review learning}
    \label{fig:rep_daily_learning}
\end{subfigure}
\hfill
\begin{subfigure}[t]{0.24\linewidth}
    \centering
    \includegraphics[width=\linewidth]{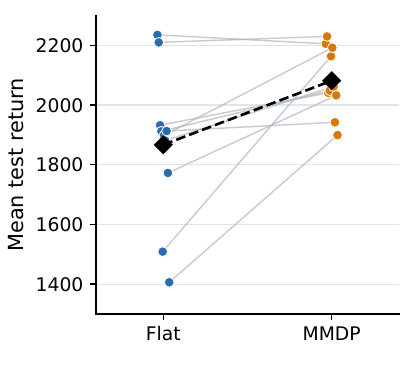}
    \caption{Daily-review test}
    \label{fig:rep_daily_test}
\end{subfigure}
\caption{Learning and generalization under two-stage and daily review. \textbf{(a,c)} Validation means and one standard deviation across ten seeds, with a symmetric-log scale in (c). \textbf{(b,d)} Seed means on 200 shared test instances per setting. Gray lines pair seeds and dashed lines connect overall means.}
\label{fig:replenishment_results}
\end{minipage}
\par\smallskip
\begin{minipage}{\linewidth}
\captionsetup{type=figure}
\centering
\begin{subfigure}[t]{0.29\linewidth}
    \centering
    \includegraphics[width=\linewidth]{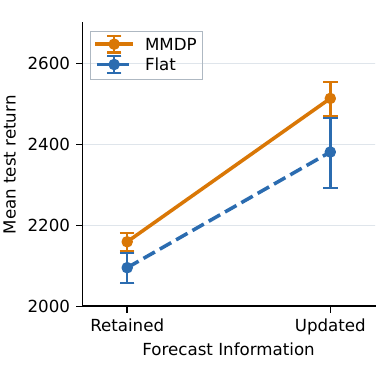}
    \caption{Forecast updating}
    \label{fig:rep_auc}
\end{subfigure}
\hfill
\begin{subfigure}[t]{0.365\linewidth}
    \centering
    \includegraphics[width=\linewidth]{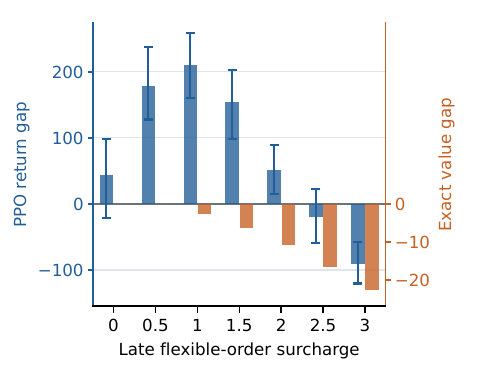}
    \caption{Late-order surcharge}
    \label{fig:rep_surcharge}
\end{subfigure}
\hfill
\begin{subfigure}[t]{0.32\linewidth}
    \centering
    \includegraphics[width=\linewidth]{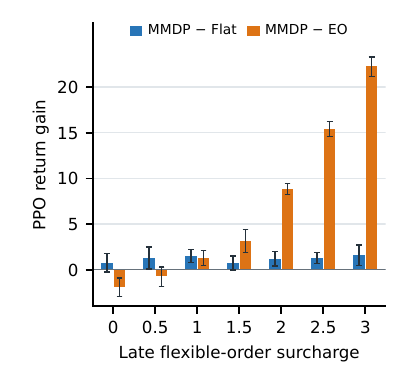}
    \caption{Model-based eligibility}
    \label{fig:rep_terminal_ppo}
\end{subfigure}
\caption{Forecast information and costly deferral in replenishment. \textbf{(a)} Test means under retained and updated forecasts. \textbf{(b)} Expiring-only (EO) minus Flat for full-episode PPO (left axis) and exact values on matched states (right axis). \textbf{(c)} PPO gains of MMDP using $\widehat{\mathcal E}_n(s)$ over Flat and EO on the same matched states used for the exact comparison in (b). Error bars show bootstrap $95\%$ CIs over ten training seeds, for marginal means in (a) and paired differences in (b,c).}
\label{fig:replenishment_revision}
\end{minipage}
\end{figure}

\paragraph{Value of refined information.}
We compare separately trained policies with and without the later refined signal. At the main setting, the refined signal correctly identifies the delivery-week demand category with $80\%$ probability. Figure~\ref{fig:rep_auc} shows mean test-return gains of $286.1$ for Flat and $354.5$ for MMDP over their respective Retained counterparts. This exploratory result suggests that information refinement can amplify the benefit of the MMDP interface. Full forecast-update results and accuracy-sensitivity analyses appear in Appendix~\ref{app:replenishment-results-analysis}.

\paragraph{Costly deferral.}
To make deferral costly, we charge a per-unit surcharge $\lambda$ on late flexible orders and tighten regular-supplier capacity. Figure~\ref{fig:rep_surcharge} compares EO with Flat. At $\lambda=1$, EO loses $2.6$ in mean exact value on matched states, while its full-episode PPO gain peaks at $209.9$. This pattern suggests that restricting the early action set can help PPO within a fixed budget even when the restriction excludes some optimal early orders. As the surcharge rises, the exact loss grows and the PPO gap reverses.

We next test whether MMDP benefits from retaining early flexible orders that EO excludes. We construct $\widehat{\mathcal E}_n(s)$ before training using the model bounds in Section~\ref{sec:semantic-interface}. We apply these action sets at every surcharge, including zero, to isolate this branch of the MMDP policy. On the same matched states as above, we compute exact optimal values for Flat, EO, and MMDP at each surcharge and train PPO separately for each interface and surcharge with a shared action representation and matched budgets (Figure~\ref{fig:rep_terminal_ppo}).

In Figure~\ref{fig:rep_terminal_ppo}, MMDP improves mean PPO return over Flat at all seven surcharges, with paired $95\%$ confidence intervals above zero at five. The exact optimal values for MMDP and EO give a crossover at $\lambda\simeq0.539$, above which some states first require early flexible ordering. Meanwhile, the mean PPO gain of MMDP over EO is negative at $\lambda=0$ and $0.5$, turns positive at $\lambda=1$, and grows at each higher tested surcharge, reaching $+22.28$ at $\lambda=3$. This consistency illustrates how the Expiring-Actions-First Principle can guide when to expand the action set beyond $\mathcal D_n$ to retain an optimal early-order option. Appendix~\ref{app:rep-costly-deferral} gives the derivation and full results.

\subsection{Synthetic Cash Management}

We next study cash management, where liquidity protection, transfer costs, and investment yield must be balanced under uncertain demand \citep{baumol1952transactions,miller1966model,pate1990dynamic,gormley2007utility,ferstl2010cash}. 
Our synthetic five- and ten-account networks contain master, investment, and operating accounts, drawing on multi-account cash-management models with networked transfers \citep{salas2021formal,bhaya2023generalized}. The two-stage daily cycle stylizes intraday liquidity forecasting and time-limited transfer windows \citep{ecb2026target}: forecasts refine and routes expire, transfers change balances immediately, and demand settles at cycle end. In the undiscounted objective, overnight deposit interest and shortfall penalties settle at day end, so $c_n(s)=0$ wherever persistent transfers and their continuations can be deferred while preserving feasibility, closing balances, and transfer charges. Environment specifications and comparison scope appear in Appendix~\ref{app:cash-environment-v7}.

\subsubsection{Five-account setting.}
We compare PPO, a random policy (Random), and a forecast-based rolling-horizon mixed-integer linear programming (MILP) controller under the Flat and MMDP interfaces. PPO uses matched observations, stage information, networks, route--amount parameterization, and training budgets. Random applies the same sampling rule over the routes and termination action available under each interface. The MILP controller is a benchmark-specific implementation of the deterministic rolling-horizon approach \citep{perkins2017stochastic}, using current forecasts and public model information; controller implementations and training settings are documented in Appendix~\ref{app:cash5-evaluation}.

\begin{figure}[!t]
    \centering
    \begin{subfigure}[t]{0.48\textwidth}
        \centering        \includegraphics[width=\textwidth]{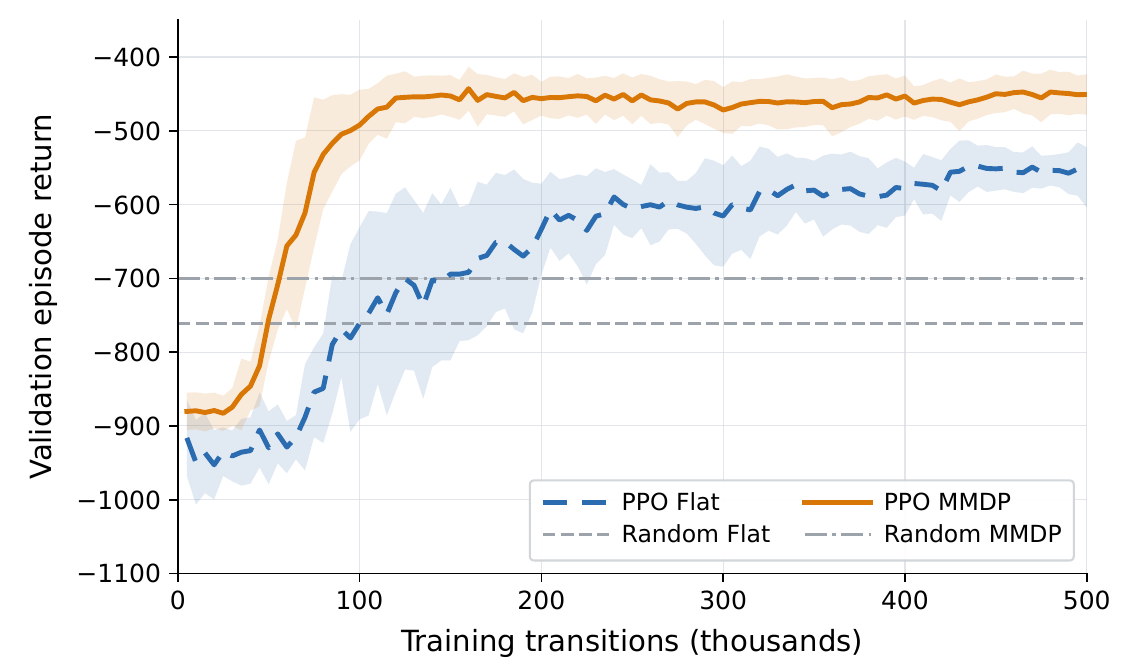}
        \caption{Five-account setting}
        \label{fig:cash_v1_a}
    \end{subfigure}
    \hfill
    \begin{subfigure}[t]{0.48\textwidth}
        \centering
        \includegraphics[width=\textwidth]{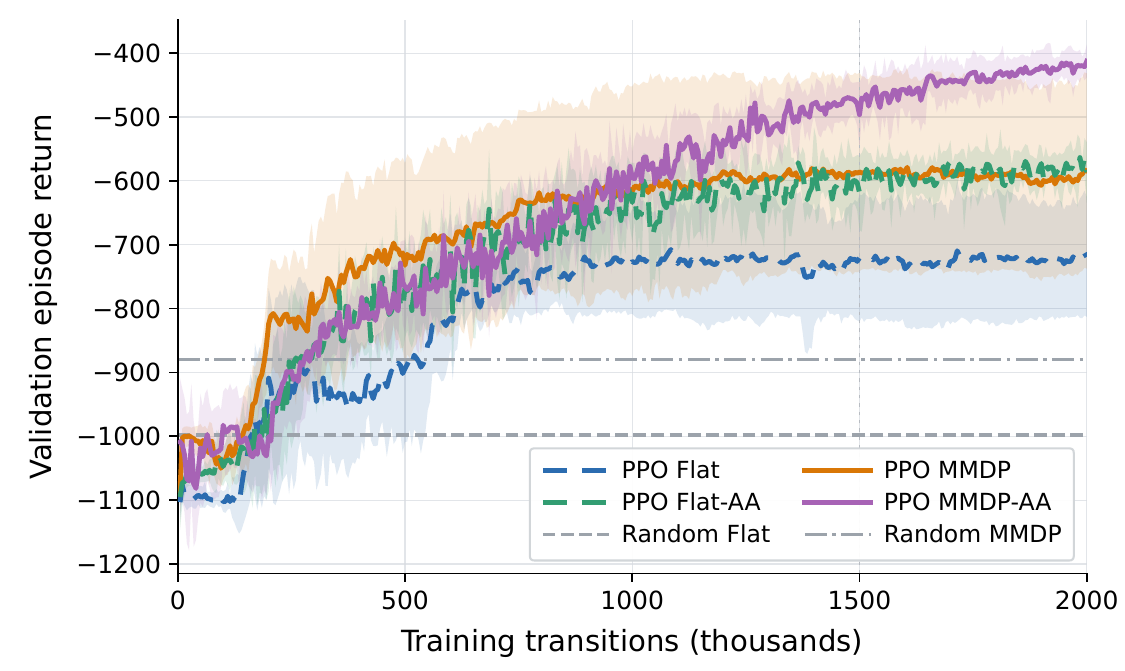}
        \caption{Ten-account setting}
        \label{fig:cash_v2_b}
    \end{subfigure}
  \caption{Raw-policy learning in cash management. Curves show mean validation returns over ten training seeds, with one standard deviation shaded.}
    \label{fig:cash_eval_curves}
\end{figure}

\begin{table}[!t]
\centering
\small
\setlength{\tabcolsep}{3pt}
\renewcommand{\arraystretch}{1.1}
\caption{Five-account results. Controller cells report mean test return on 200 shared instances, with controller time in parentheses below (ms/episode; 20 shared instances, single-thread CPU). Mask averages five equal-size funding-preserving masks. Contrast cells report paired return differences with 95\% CIs below.}
\label{tab:cash5-independent-test}
\begin{tabular*}{\linewidth}{@{\extracolsep{\fill}}lccccc@{}}
\toprule
& \multicolumn{3}{c}{Interface comparison} & \multicolumn{2}{c}{Mask comparison} \\
\cmidrule(lr){2-4}\cmidrule(l){5-6}
Controller & \shortstack{Flat\\{\scriptsize (ms/episode)}} & \shortstack{MMDP\\{\scriptsize (ms/episode)}} & \shortstack{MMDP $-$ Flat\\{\scriptsize [95\% CI]}} & \shortstack{Mask\\{\scriptsize (ms/episode)}} & \shortstack{MMDP $-$ Mask\\{\scriptsize [95\% CI]}} \\
\midrule
PPO & \shortstack{$-552.003$\\{\scriptsize $(0.581)$}} & \shortstack{$-434.278$\\{\scriptsize $(0.543)$}} & \shortstack{$+117.725$\\{\scriptsize $[85.237,151.115]$}} & \shortstack{$-486.853$\\{\scriptsize $(0.610)$}} & \shortstack{$+52.575$\\{\scriptsize $[41.175,64.955]$}} \\
Random & \shortstack{$-789.646$\\{\scriptsize $(0.190)$}} & \shortstack{$-720.130$\\{\scriptsize $(0.160)$}} & \shortstack{$+69.515$\\{\scriptsize $[64.246,74.890]$}} & \shortstack{$-721.888$\\{\scriptsize $(0.160)$}} & \shortstack{$+1.758$\\{\scriptsize $[-1.107,4.687]$}} \\
MILP & \shortstack{$-330.286$\\{\scriptsize $(2.255)$}} & \shortstack{$-330.286$\\{\scriptsize $(1.988)$}} & \shortstack{$0.000$\\{\scriptsize $[0.000,0.000]$}} & \shortstack{$-330.286$\\{\scriptsize $(2.082)$}} & \shortstack{$0.000$\\{\scriptsize $[0.000,0.000]$}} \\
\bottomrule
\end{tabular*}
\end{table}

\paragraph{Learning and independent-test performance.}
Figure~\ref{fig:cash_v1_a} shows an MMDP--Flat separation during PPO training. On independent test instances (Table~\ref{tab:cash5-independent-test}), MMDP improves PPO return over Flat by $117.725$ (95\% CI $[85.237,151.115]$). 
Random also improves by $69.515$ (95\% CI $[64.246,74.890]$). This non-learning gain shows that the MMDP interface can improve outcomes without policy training in this benchmark. MILP attains the highest mean test return ($-330.286$) and selects identical transfers under Flat and MMDP. Thus, the MMDP deferral restriction causes no observed loss for this controller on this test bank, but this is not a general losslessness guarantee. On the separate 20-instance single-thread CPU subset, measured controller time per episode is lower under MMDP by approximately $6.5\%$ for PPO, $15.8\%$ for Random, and $11.8\%$ for MILP. 

\paragraph{Route selection beyond action count.}
To test whether action count alone explains the learned interface gain, we compare MMDP with five equal-size masks that preserve all three direct funding routes.  In the Mask comparison of Table~\ref{tab:cash5-independent-test}, MMDP PPO gains $52.575$ (95\% CI $[41.175,64.955]$) over the mask average, whereas Random shows no clear separation ($+1.758$, 95\% CI $[-1.107,4.687]$). Thus, for PPO in this benchmark, matching action count and direct-funding coverage does not eliminate MMDP's advantage; the comparison does not isolate expiration semantics from other route and decoding differences.

\paragraph{Continuous-selector ablation.}
We retrain MMDP with Flat's route bins, STOP threshold, and amount decoder, conditioning early-stage sampling on eligible routes. Against the same frozen Flat policies, the PPO gap shrinks from $117.725$ to $53.877$ (95\% CI $[19.774,89.389]$). The original encoding contributes to the larger gain, and MMDP still outperforms Flat with the aligned selector. Evaluation protocols and interpretation limits appear in Appendix~\ref{app:cash5-semantic-mask}.

\subsubsection{Ten-account setting.}

Compared with the five-account setting, the ten-account network expands the early-stage route catalogue from 14 to 74 routes, including 16 rather than 6 routes that expire at the stage boundary. We use this larger action-selection space to examine action abstraction (AA; Section~\ref{sec:action-execution}), policy-guided search (Section~\ref{sec:policy-guided-search}), and planning-budget sensitivity. Controller implementations and training settings appear in Appendix~\ref{app:cash10-methods}.

\paragraph{Learning with action abstraction.}
We train Flat, MMDP, Flat-AA, and MMDP-AA with matched observations, hidden-layer widths, PPO settings, and training budgets. Figure~\ref{fig:cash_v2_b} shows two distinct learning patterns. Without AA, MMDP starts above Flat and maintains a validation-return advantage throughout training. With AA, the two policies begin at similar returns, but MMDP-AA pulls ahead as training proceeds. One possible interpretation is that AA initially eases action execution for both interfaces, while the benefit of selecting among stage-appropriate opportunities becomes more visible as the policies learn.

\begin{table}[t]
\centering
\small
\setlength{\tabcolsep}{3pt}
\renewcommand{\arraystretch}{1.05}
\caption{Ten-account Raw MMDP-minus-Flat test gaps. Paired 95\% CIs resample training seeds for PPO and instances for Random/MILP; Random averages ten action draws per instance.}
\label{tab:cash10-interface-by-controller}
\begin{tabular*}{\linewidth}{@{\extracolsep{\fill}}cccc@{}}
\toprule
PPO & PPO-AA & Random & MILP \\
\midrule
$+145.09\;{\scriptstyle[45.27,228.85]}$ & $+158.63\;{\scriptstyle[143.50,173.97]}$ & $+128.64\;{\scriptstyle[121.29,136.12]}$ & $+10.02\;{\scriptstyle[6.58,13.85]}$ \\
\bottomrule
\end{tabular*}
\end{table}

\paragraph{Test interface gaps.}
Table~\ref{tab:cash10-interface-by-controller} reports Raw MMDP-minus-Flat test gaps under each controller. The gains are $145.09$ for PPO and $158.63$ for PPO-AA, showing that the MMDP advantage persists with and without AA on the held-out test instances. Random also improves by $128.64$. Unlike in the five-account setting, MMDP raises the ten-account MILP's mean test return by $10.02$ (95\% CI $[6.58,13.85]$). Since MILP obtains the highest mean returns of  $-206.29$ under MMDP, it shows that MMDP can also add value to a strong planning baseline.

\paragraph{Policy-guided search.}
We apply policy-guided search to the same frozen final policies on the same test instances, without further training or checkpoint selection. Candidate-generation rules and candidate-count limits are shared within each AA and non-AA interface pair. Search executes the selected action and replans under the same per-decision simulation cap for all four policy families, although realized planning work may differ.

\begin{table}[htbp]
\centering
\small
\setlength{\tabcolsep}{3pt}
\renewcommand{\arraystretch}{1.15}
\caption{Ten-account independent-test mean returns for the same frozen policies on 200 shared instances. Raw and Search brackets show marginal bootstrap 95\% CIs over ten training seeds; Search-gain brackets show paired-bootstrap 95\% CIs for within-seed differences. Search uses the primary planning budget. Higher return is better.}
\label{tab:cash-complete-comparison}
\begin{tabular*}{\linewidth}{@{\extracolsep{\fill}}lccr@{}}
\toprule
Policy & Raw & Search & \multicolumn{1}{c@{}}{Search gain} \\
\midrule
Flat & $-712.64\;{\scriptstyle[-765.65,-654.86]}$ & $-604.80\;{\scriptstyle[-656.28,-547.31]}$ & $+107.84\;{\scriptstyle[94.73,120.96]}$ \\
MMDP & $-567.55\;{\scriptstyle[-665.37,-509.30]}$ & $-454.69\;{\scriptstyle[-517.07,-414.68]}$ & $+112.85\;{\scriptstyle[87.24,151.90]}$ \\
Flat-AA & $-546.93\;{\scriptstyle[-564.09,-531.37]}$ & $-373.88\;{\scriptstyle[-381.99,-365.32]}$ & $+173.05\;{\scriptstyle[156.05,191.69]}$ \\
MMDP-AA & $-388.30\;{\scriptstyle[-395.33,-381.45]}$ & $-293.69\;{\scriptstyle[-298.21,-289.54]}$ & $+94.61\;{\scriptstyle[88.65,100.74]}$ \\
\bottomrule
\end{tabular*}
\end{table}

Table~\ref{tab:cash-complete-comparison} shows Search gains for all four families, positive in every training-seed pair. MMDP-AA with Search achieves the highest mean return among the four evaluated learned policy families.

\paragraph{Planning-budget sensitivity.}
We vary the per-decision planning cap over $1\times$, $2\times$, and $4\times$ a fixed base budget, keeping the policies, test instances, and candidate-generation rules unchanged. Figure~\ref{fig:cash-search-frontier} shows substantial gains for both AA controllers at $1\times$, but no consistent further improvement as the cap increases. MMDP-AA has the highest mean return among the four policy families at each evaluated cap. Complete ten-account evaluation protocols, paired comparisons, budget results, and computational accounting appear in Appendix~\ref{app:cash10-evaluation}.

\begin{figure}[t]
\centering
\begin{subfigure}[t]{0.5\linewidth}
\centering
\includegraphics[width=\linewidth]{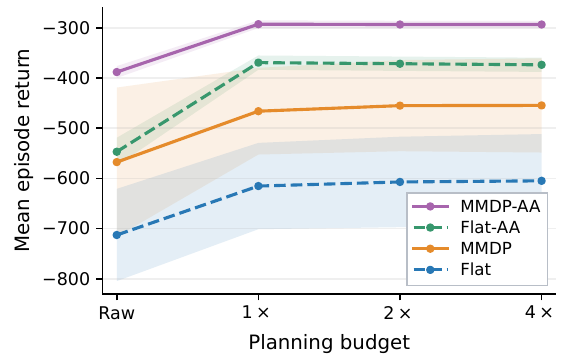}
\caption{Planning-budget sensitivity}
\label{fig:cash-search-frontier}
\end{subfigure}
\hfill
\begin{subfigure}[t]{0.48\linewidth}
\centering
\includegraphics[width=\linewidth]{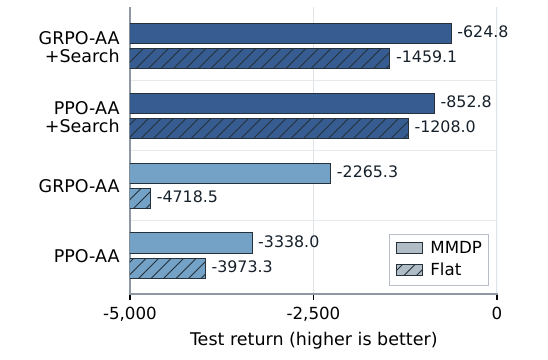}
\caption{Production-scale comparison}
\label{fig:cash_v3_final}
\end{subfigure}

\caption{\textbf{(a)} Ten-account returns under caps of 576, 1,152, and 2,304 simulated transitions per decision ($1\times$, $2\times$, and $4\times$). Raw uses no Search. Shading shows across-seed SD. \textbf{(b)} Five-seed mean returns in the 27-account, 24-stage simulator. Bars start at $-5{,}000$.}
\label{fig:cash-budget-production}
\end{figure}

\subsection{Production-Scale Cash Management}
\label{sec:production_scale_cash}

We next examine the methods in a production-scale cash-management system with 27 accounts and 24 hourly decision stages per day. Unlike the two-stage synthetic benchmarks, this setting involves hourly updates to balances, forecasts, and network availability. We combine simulator comparisons, a pre-deployment historical backtest, and observations from human-reviewed deployment to examine performance, estimated economic value, and operational uptake. The production simulator's environment, evaluation protocol, and results appear in Appendix~\ref{app:prod_cashflow_env}.

\paragraph{Simulator evaluation.}
We compare Flat and MMDP interfaces with PPO-AA and GRPO-AA, each with and without Search. GRPO-AA forms group-relative advantages from trajectories sharing an initial condition, without a learned value baseline \citep{shao2024deepseekmath}; Search scores candidate plans under the current forecast and replans after new observations. Figure~\ref{fig:cash_v3_final} reports five-seed means on a shared simulator test set generated from held-out operational source data. MMDP has higher means in all four matched pipeline comparisons. Under Flat, PPO-AA exceeds GRPO-AA both Raw and with Search; under MMDP, this ordering reverses, and Search-assisted MMDP-GRPO-AA attains the highest mean. This reversal suggests that the relative performance of PPO-AA and GRPO-AA depends on the decision interface in this simulator.

\paragraph{Pre-deployment backtest and human-reviewed deployment.}
A historical backtest spanning six calendar months and 94 eligible business days compares the Search-assisted MMDP-GRPO-AA controller with the executed incumbent workflow, including human adjustments, under the same external cash flows and constraints. Mean daily relative net-income improvement is $8.25\%$ (approximate 95\% CI: $[3.80\%,12.70\%]$), where net income comprises overnight deposit interest minus shortfall financing costs and penalties.

After deployment of a separate operational model, mean daily amount-based adoption increases from $23.92\%$ in the pre-deployment period to $73.93\%$ across 135 eligible days in five complete months and part of a sixth. The before--after comparison does not isolate a causal deployment effect. Model scope, definitions, monthly results, and aggregation details appear in Appendix~\ref{app:production-backtest}.

\subsection{Maze Navigation}
\label{sec:maze-navigation}

To complement the replenishment and cash-management experiments, we instantiate the maze example in Figure~\ref{fig:intro_asymmetry} as a gridworld navigation task \citep{chevalierboisvert2023minigrid,sutton2018reinforcement,samvelyan2021minihack}. The agent seeks treasure as clues narrow the candidate locations and rockfalls permanently close rooms (Appendix~\ref{app:maze-environment}). A room action moves one tile toward its target. The stage-termination action $a_n^{\mathrm{end}}$ advances the stage and reveals the next clue.

\noindent\begin{minipage}{\linewidth}
\captionsetup{type=figure}
\centering
\begin{subfigure}[t]{0.18\linewidth}
\centering
\includegraphics[width=\linewidth]{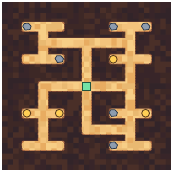}
\caption{maze-16}
\end{subfigure}
\hfill
\begin{subfigure}[t]{0.305\linewidth}
\centering
\includegraphics[width=\linewidth]{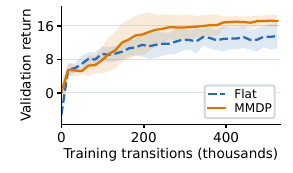}
\caption{maze-16 learning}
\end{subfigure}
\hfill
\begin{subfigure}[t]{0.18\linewidth}
\centering
\includegraphics[width=\linewidth]{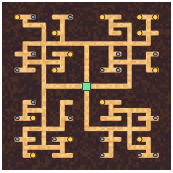}
\caption{maze-48}
\end{subfigure}
\hfill
\begin{subfigure}[t]{0.305\linewidth}
\centering
\includegraphics[width=\linewidth]{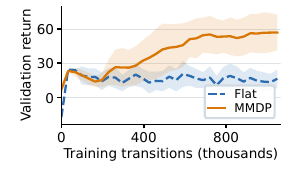}
\caption{maze-48 learning}
\end{subfigure}
\caption{After the first rockfall, gray stones mark closures, gold dots mark treasure candidates, and green squares mark the start. Curves show validation means and one SD over ten seeds.}
\label{fig:maze-learning}
\end{minipage}

Figure~\ref{fig:maze-learning} shows higher late-training returns for MMDP in maze-16 (three stages) and maze-48 (five stages). Its test gains over Flat are $+2.64$ and $+37.64$, respectively (Table~\ref{tab:maze-test}). In maze-48, greedy Flat policies take fewer first-stage moves ($130$ vs.\ $229$), end the stage early more often ($70.0\%$ vs.\ $8.8\%$), and collect fewer chests while rewards are higher ($0.23$ vs.\ $0.56$ per episode). Appendix~\ref{app:maze-evaluation} gives the protocol and behavioral results.

\section{Conclusion}
\label{sec:conclusion}

Maturing Markov Decision Processes make commitment timing explicit as information refines and opportunities expire. The Expiring-Actions-First Principle gives a sufficient condition for lossless deferral. When deferral of the entire persistent set cannot be certified, model bounds identify individual actions that can safely be deferred. The resulting stage-local interface can be combined with action abstraction and policy-guided search.

In staged replenishment, MMDP improves finite-budget learning and held-out performance. Forecast updating helps both interfaces, and the primary setting suggests that MMDP can amplify the benefit of refined information. With costly deferral, MMDP's advantage over fixed deferral emerges in line with the analysis of when early ordering becomes necessary, illustrating how the Expiring-Actions-First Principle can guide early-order eligibility decisions. In synthetic cash management, MMDP improves PPO, Random, and ten-account MILP results. Equal-size mask controls show that route count and direct-funding coverage do not fully explain the five-account PPO gain. The evaluated AA encodings retain an MMDP advantage, and Search improves all four frozen policy families. The production case adds descriptive simulator, backtest, and human-reviewed adoption evidence. In maze navigation, the MMDP advantage is larger in the more complex of the two settings.

Certifying deferral for individual actions depends on the tightness of the available bounds. In the production system, network-coupled transfers required a system-wide transition rather than an account-level A/B test; the deployment observations therefore describe adoption, not a causal revenue effect. Future work includes testing the MMDP interface across a wider range of information structures and networks and exploring its use with optimization-based controllers in production.

\subsection*{AI use statement}

We utilize tools based on large language models to assist with literature searches, language editing, refining theoretical formulations and proofs, experimental diagnosis, implementation and debugging, and the creation of figures and tables. Controlled benchmark instances are sampled by programmatic simulators; AI-assisted work also includes the development of experimental software and the verification of controlled experiments. These verifications do not replace the authors' validation of mathematical arguments, references, and the interpretation of results. Authors bear full responsibility for the final manuscript, including all text, assertions, code, and experimental results generated with the assistance of artificial intelligence.

\subsection*{Ethics statement}

The controlled replenishment, cash-management, and maze experiments use synthetic environments. Production evidence is presented through aggregate results and system descriptions; raw operational records and identifying account or institution details are not released. The deployed workflow retains human review of model-generated recommendations. Retrospective backtest gains and observed recommendation adoption should not be interpreted as causal evidence of increased realized revenue. Errors in forecasts, models, or optimization can cause liquidity shortfalls and financial losses. Operational use therefore requires application-specific validation, enforceable feasibility and risk constraints, human oversight, and ongoing monitoring; performance in the studied environments does not establish safety or effectiveness in other operational settings.

\subsection*{Reproducibility statement}

Research code for the synthetic benchmarks is available at \url{https://github.com/scujx/mmdp}. The accompanying release is \href{https://github.com/scujx/mmdp/releases/tag/v1.1.0}{\texttt{v1.1.0}}. It includes training and evaluation code, selected synthetic seed-level results, and the script for Figure~\ref{fig:rep_terminal_ppo}. The repository uses a non-commercial research license.

The assumptions and proof of the Expiring-Actions-First Principle are provided in Appendix~\ref{app:expiring_priority_proof}, and implementation details are given in Appendix~\ref{app:structure-aware-details}. The replenishment environment, training protocol, and complete results are documented in Appendix~\ref{app:replenishment-controlled-evaluation}. Cash-management environment specifications appear in Appendix~\ref{app:cash-environment-v7}, with five-account methods and evaluations and ten-account learning, Search, and budget protocols in Appendices~\ref{app:cash5-evaluation}, \ref{app:cash5-semantic-mask}, \ref{app:cash10-learning-dynamics}, \ref{app:cash-search-evaluation}, and \ref{app:cash-planning-budget}. The cash-management appendix also documents the ten-account Random and MILP references and serial runtime measurements. Appendix~\ref{app:maze-evaluation} specifies the maze environment and evaluation. These descriptions distinguish training, validation, and independent testing, and specify the replication units used for uncertainty estimation. The complete production-derived simulator and underlying operational records cannot be released because of confidentiality restrictions. Appendix~\ref{app:prod_cashflow_env} describes that setting, but the production results are not fully independently reproducible from the disclosed materials.

\subsubsection*{Acknowledgments}

We thank the anonymous reviewers for their thoughtful comments and questions, which helped us clarify the conditions for safe deferral and extend our experiments to costly deferral and multi-stage settings.

\bibliography{iclr2027_conference}
\bibliographystyle{iclr2027_conference}

\appendix

\section{Additional Details on MMDP Structure}
\label{app:mmdp_structure}

\subsection{Interpretation of Non-Degenerate Stages}
\label{sec:nondegenerate_stage}

Remark~\ref{rem:nondegenerate_stage} specifies which boundaries define distinct effective stages in the MMDP abstraction. A new effective stage requires both strict information refinement and the permanent expiration of at least one domain-action opportunity. A transition exhibiting only one component does not define a separate effective stage; its events, rewards, and state transitions remain part of the underlying MDP. Treating such events as belonging to the same effective stage does not mean deleting them or replacing several underlying transitions with a single transition.

This partition does not imply that waiting is optimal whenever opportunities remain open, or that immediate execution is optimal whenever no new information arrives. Rewards, resource availability, and other state consequences may still depend on timing. The partition identifies the joint maturation structure, while Theorem~\ref{thm:expiring_priority} gives a separate, conditional criterion for deferring persistent actions.

\subsection{Details and Proof of the Expiring-Actions-First Principle}
\label{app:expiring_priority_proof}

This subsection makes the comparison set and delay-cost definition in Section~\ref{sec:mmdp:priority} explicit and proves Theorems~\ref{thm:expiring_priority} and~\ref{thm:action-wise-deferral}. Fix a stage $n<N$ and a state $s$ satisfying the local comparison conditions. Assume bounded returns, attained optimal continuations, and attained measurable maxima, so that the values and decision alternatives below are well defined.

\paragraph{The common comparison set.}
Structural action sets record unexpired opportunities, while executability may additionally depend on the current state. Define
\begin{equation}
B=B_{n+1}(s)
:=
\left\{
a\in\mathcal A_{n+1}:
a\text{ is executable at }s
\right\}.
\label{eq:priority_comparison_set}
\end{equation}
For product refinement, $\mathcal S_{n+1}=\mathcal S_n\times\mathcal Z_{n+1}$, where $\mathcal Z_{n+1}$ is the space of new information. The local feasibility assumption requires $B$ to be nonempty and every action in $B$ to remain executable on a probability-one set of successor states. All three quantities $J_{\mathrm{now}}$, $J_{\mathrm{delay}}$, and $J_{\mathrm{adapt}}$ use this same set. Observing $z$ allows the adaptive alternative to select among these actions, but does not enlarge the comparison set.

The full next-stage value $V_{n+1}^\star$ may optimize over additional actions that become executable only at the successor state.

\paragraph{Transitions and value scales.}
Executing $a_n^{\mathrm{end}}$ produces $(s,z)\sim\Gamma(\cdot\mid s)$ in one underlying MDP transition, as assumed in the local comparison. All expectations below use this distribution at the fixed current state $s$. Write $r_e=r_n(s,a_n^{\mathrm{end}})$ for the expected immediate reward of ending the stage. The value $Q_{n+1}^\star((s,z),a)$ is measured from the next-stage decision state; multiplying it by $\gamma$ and adding $r_e$ expresses the delayed alternative on the same scale as $Q_n^\star(s,a)$.

The next-stage action value specifies only the first next-stage action. Subsequent decisions use the full model and may depend on the revealed information. Thus, $J_{\mathrm{delay}}$ selects its first next-stage action before observing $z$, while $J_{\mathrm{adapt}}$ selects that action after observing $z$. Both permit optimal continuation thereafter.

The three alternatives are
\begin{align*}
J_{\mathrm{now}}
&:=
\max_{a\in B}Q_n^\star(s,a),
\\
J_{\mathrm{delay}}
&:=
r_n(s,a_n^{\mathrm{end}})
+\gamma\max_{a\in B}
\mathbb E_z\!\left[Q_{n+1}^\star((s,z),a)\right],
\\
J_{\mathrm{adapt}}
&:=
r_n(s,a_n^{\mathrm{end}})
+\gamma\mathbb E_z\!\left[
\max_{a\in B}Q_{n+1}^\star((s,z),a)
\right].
\end{align*}

\paragraph{Information gain.}
The two delayed alternatives share the same stage-termination transition and reward. Consequently,
\begin{align}
\delta_n(s)
&=
J_{\mathrm{adapt}}-J_{\mathrm{delay}}
\nonumber\\
&=
\gamma\left[
\mathbb E_z\!\left[
\max_{a\in B}Q_{n+1}^\star((s,z),a)
\right]
-
\max_{a\in B}
\mathbb E_z\!\left[
Q_{n+1}^\star((s,z),a)
\right]
\right].
\label{eq:priority_information_expansion}
\end{align}
For every fixed $a\in B$, the pointwise maximum inside the first expectation is at least $Q_{n+1}^\star((s,z),a)$. Taking expectations and then maximizing over the fixed action gives $\delta_n(s)\geq0$. This quantity measures the gain from adapting the first next-stage action, rather than the net advantage of waiting over acting now. Strict information refinement need not yield a strictly positive gain: the same action may remain optimal across all relevant realizations.

\paragraph{Delay cost.}
For a fixed action $a\in B$, the expected discounted return from executing it immediately and continuing optimally is $Q_n^\star(s,a)$. The return from first ending the stage, then executing the same $a$ as the first next-stage action, and continuing optimally is
\[
r_e+\gamma\mathbb E_z\!\left[
Q_{n+1}^\star((s,z),a)
\right].
\]
Define the signed return loss from delaying this fixed action as
\begin{equation}
\ell_n(s,a)
:=
Q_n^\star(s,a)
-
\left[
r_e+\gamma\mathbb E_z\!\left[
Q_{n+1}^\star((s,z),a)
\right]
\right].
\label{eq:priority_action_delay_loss}
\end{equation}
This loss may be negative if delaying the fixed action improves its return. A valid delay-cost bound is any finite $c_n(s)\geq0$ satisfying
\begin{equation}
c_n(s)
\geq
\sup_{a\in B}\bigl[\ell_n(s,a)\bigr]_+,
\qquad [u]_+:=\max\{u,0\}.
\label{eq:priority_uniform_delay_cost}
\end{equation}
This is the precise meaning of the uniform upper bound used in the main text.

For every $a\in B$, the bound implies
\[
Q_n^\star(s,a)
\leq
r_e+\gamma\mathbb E_z\!\left[
Q_{n+1}^\star((s,z),a)
\right]+c_n(s).
\]
Maximizing over the same set $B$ gives
\begin{align}
J_{\mathrm{now}}
&\leq
r_e+\gamma
\max_{a\in B}
\mathbb E_z\!\left[
Q_{n+1}^\star((s,z),a)
\right]+c_n(s)
\nonumber\\
&=
J_{\mathrm{delay}}+c_n(s),
\end{align}
which establishes \eqref{eq:fixed_action_delay_bound}.

The loss in \eqref{eq:priority_action_delay_loss} includes the complete effect of delay on discounted return, including changes in rewards, resources, state consequences, and subsequent opportunities. A model-specific bound on these effects is admissible whenever it satisfies \eqref{eq:priority_uniform_delay_cost}. An immediate execution-fee difference alone need not provide such a bound. In particular, the absence of an explicit waiting fee does not imply $c_n(s)=0$; zero is valid when delaying every action in $B$ weakly preserves its optimal return. The existence of a bound does not guarantee that it satisfies the theorem's threshold $c_n(s)\leq\delta_n(s)$.

In terminal-commitment problems, a commitment ends the episode, so $Q_n^\star(s,a)=r_n(s,a)$ for a commitment action. Both the delay loss in Eq.~\eqref{eq:priority_action_delay_loss} and the information gain in Eq.~\eqref{eq:priority_information_expansion} can then be computed from the stage payoffs and the signal law.

\paragraph{The value of deferral.}
Using delayed execution of a preselected action as the common reference gives
\begin{align}
J_{\mathrm{adapt}}-J_{\mathrm{now}}
&=
\bigl(J_{\mathrm{adapt}}-J_{\mathrm{delay}}\bigr)
-
\bigl(J_{\mathrm{now}}-J_{\mathrm{delay}}\bigr)
\nonumber\\
&\geq
\delta_n(s)-c_n(s).
\label{eq:priority_common_reference}
\end{align}
All rewards and costs are already included in the three returns. The quantity $\delta_n(s)$ measures the advantage of adaptive selection over delayed fixed-action execution, while $c_n(s)$ bounds the advantage of immediate execution over that same reference. Thus, this decomposition does not charge any cost twice. Since $c_n(s)$ is an upper bound, the final relation is generally an inequality rather than an equality.

\begin{proof}[Proof of Theorem~\ref{thm:expiring_priority}]
The Bellman equation for the executable stage-termination action gives
\begin{align}
Q_n^\star(s,a_n^{\mathrm{end}})
&=
r_e+\gamma
\mathbb E_z\!\left[
V_{n+1}^\star((s,z))
\right]
\nonumber\\
&\geq
r_e+\gamma
\mathbb E_z\!\left[
\max_{a\in B}Q_{n+1}^\star((s,z),a)
\right]
\nonumber\\
&=
J_{\mathrm{adapt}}.
\label{eq:priority_end_dominates_adaptive}
\end{align}
The inequality follows because every action in $B$ is feasible on the specified probability-one set of successor states, while $V_{n+1}^\star$ optimizes over all executable next-stage actions.

If $c_n(s)\leq\delta_n(s)$, equation~\eqref{eq:priority_common_reference} implies
\[
Q_n^\star(s,a_n^{\mathrm{end}})
\geq J_{\mathrm{adapt}}
\geq J_{\mathrm{now}}
=
\max_{a\in B}Q_n^\star(s,a).
\]
Thus, ending the stage weakly dominates every currently executable persistent action. Since
\[
\mathcal A_n=\mathcal D_n\mathbin{\dot\cup}\mathcal A_{n+1}
\qquad\text{and}\qquad
a_n^{\mathrm{end}}\in\mathcal D_n,
\]
maximization over actions executable at $s$ yields
\[
\max_{a\in\mathcal A_n}Q_n^\star(s,a)
=
\max_{a\in\mathcal D_n}Q_n^\star(s,a).
\]
An attained maximizer on the right is therefore an optimal decision at $s$, establishing the stated choice of $\pi_n^\star(s)$.
\end{proof}

\paragraph{Local loss bound.}
Let $M_D$ be the maximum full-model action value over executable actions in $\mathcal D_n$. Restricting the current decision to $\mathcal D_n$ while allowing full-model continuation gives the exact local loss
\begin{equation}
V_n^\star(s)-M_D=\max\{0,J_{\mathrm{now}}-M_D\}.
\label{eq:priority_exact_local_loss}
\end{equation}
The loss is zero exactly when $J_{\mathrm{now}}\le M_D$. More generally, retaining an expiring maximizer makes it safe to discard any persistent action whose value is at most $M_D$.

Stage termination belongs to the executable expiring set, so \eqref{eq:priority_end_dominates_adaptive} gives $M_D\ge J_{\mathrm{adapt}}$. Combining this inequality with \eqref{eq:priority_common_reference} yields
\begin{equation}
J_{\mathrm{now}}-M_D
\le J_{\mathrm{now}}-J_{\mathrm{adapt}}
\le c_n(s)-\delta_n(s).
\end{equation}
For $c_n(s)>\delta_n(s)$, substituting this inequality into \eqref{eq:priority_exact_local_loss} gives
\begin{equation}
0\le V_n^\star(s)-M_D\le c_n(s)-\delta_n(s).
\end{equation}

\paragraph{The expanded action set.}
Equation~\eqref{eq:action-wise-deferral} defines $\mathcal E_n(s)$ by retaining the executable expiring set and adding each $a\in B$ whose current value exceeds $M_D$. Write $W_n(s,a)=r_e+\gamma\mathbb E_z[Q_{n+1}^\star((s,z),a)]$ for the value of deferring $a$. Since $Q_n^\star(s,a)=W_n(s,a)+\ell_n(s,a)$ and $J_{\mathrm{adapt}}=J_{\mathrm{delay}}+\delta_n(s)$, a persistent action enters $\mathcal E_n(s)$ exactly when
\begin{equation}
\ell_n(s,a)>
\bigl[\delta_n(s)+J_{\mathrm{delay}}-W_n(s,a)\bigr]
+\bigl[M_D-J_{\mathrm{adapt}}\bigr].
\label{eq:action-advantage-over-waiting}
\end{equation}
The first bracket is the gain from waiting and adapting instead of waiting to execute $a$. The second credits any additional value of the best expiring action over waiting and adapting. Because $J_{\mathrm{delay}}\geq W_n(s,a)$ and $M_D\geq J_{\mathrm{adapt}}$, no persistent action enters if $c_n(s)\leq\delta_n(s)$.

\begin{proof}[Proof of Theorem~\ref{thm:action-wise-deferral}]
Every executable action outside $\mathcal E_n(s)$ has value at most $M_D$. The set $\mathcal E_n(s)$ retains an expiring action that attains $M_D$. Thus, over executable actions,
\[
\max_{a\in\mathcal A_n}Q_n^\star(s,a)
=\max_{a\in\mathcal E_n(s)}Q_n^\star(s,a).
\]
An attained maximizer on the right is therefore optimal in the full model. We may select
\[
\pi_n^\star(s)\in\arg\max_{a\in\mathcal E_n(s)}Q_n^\star(s,a).
\qedhere
\]
\end{proof}

\paragraph{Scope of the conclusion.}
The theorem is local to the specified state and stage. If its conditions hold at every state where the restriction is imposed, the finite-horizon Bellman argument permits an optimal policy satisfying these restrictions, under the usual measurable-selection assumptions. Other optimal policies may select persistent actions when values tie. If no persistent action is currently executable, the restriction is immediate and requires no information-gain comparison.

\section{Implementation Details for Structure-Aware Reinforcement Learning}
\label{app:structure-aware-details}

This section explains the stage-local interface, action abstraction, and execution-time Search. It separates the semantic restriction from the policy encoding, amount executor, and planning approximation.

\subsection{Semantic Interfaces and Policy Implementations}
\label{app:semantic-interface-details}

\paragraph{Stage-local decisions.}
The first branch of the MMDP policy and the fixed EO policy use the same action set at nonfinal stages. Both use the remaining executable actions at the final stage:
\begin{equation}
\begin{aligned}
\pi_\theta(\{a\in\mathcal D_n:a\text{ is executable at }s\}\mid s,n)&=1, &&n<N,\\
\pi_\theta(\{a\in\mathcal A_N:a\text{ is executable at }s\}\mid s,N)&=1.
\end{aligned}
\label{eq:stage-local-admissibility}
\end{equation}
The stage-termination action $a_n^{\mathrm{end}}\in\mathcal D_n$ remains a choice alongside expiring domain actions. Selecting it advances through $\Gamma$ and forfeits unresolved expiring opportunities. Flat instead permits all currently executable actions in $\mathcal A_n$, including persistent actions. Structural membership describes opportunity expiration; executability additionally depends on the current state.

\paragraph{Relation to the theorem.}
For the first branch of Eq.~\eqref{eq:stage-local-policy}, let $L_{\mathrm{adapt}}(s)$ be the value of a feasible signal-adaptive continuation whose first action after observing the signal belongs to the comparison set $B$ in Eq.~\eqref{eq:priority_comparison_set}. Let $U_{\mathrm{delay}}(s)$ upper-bound the best delayed plan that fixes its persistent action before the signal. The latter bound can relax the continuation after each fixed action, provided it bounds the maximum over all such actions. Then $\max\{0,L_{\mathrm{adapt}}(s)-U_{\mathrm{delay}}(s)\}$ lower-bounds $\delta_n(s)$. If this quantity is at least a valid upper bound on $c_n(s)$, Theorem~\ref{thm:expiring_priority} certifies the expiring-only support under its other assumptions. Under an undiscounted objective, a postponement that preserves feasibility and payoff for every persistent action and its continuation certifies $c_n(s)=0$. The fixed interface in the PPO surcharge sweep always applies the expiring-only support, including where the sufficient condition fails.

\paragraph{State-dependent screening.}
Let $L_D(s)$ lower-bound the best expiring value $M_D$ by evaluating any feasible plan that starts with an expiring action. For each executable persistent action $a$, let $U_n(s,a)$ upper-bound $Q_n^\star(s,a)$ by fixing the first action to $a$ and relaxing its continuation, for example by allowing later decisions to use additional observations. With the convention that all action sets are restricted to executable actions, the second branch of Eq.~\eqref{eq:stage-local-policy} uses
\begin{equation}
\widehat{\mathcal E}_n(s)=\mathcal D_n\cup\{a\in\mathcal A_{n+1}: U_n(s,a)>L_D(s)\}.
\label{eq:action-wise-policy}
\end{equation}
A discarded action has value at most $L_D(s)\le M_D$, and the screen keeps all expiring actions. Thus $\mathcal E_n(s)\subseteq\widehat{\mathcal E}_n(s)$, and the sets coincide with exact bounds. In the other direction, a feasible plan starting with a persistent action can exceed an upper bound on $M_D$. Such a comparison certifies that every optimal current choice is persistent. Appendix~\ref{app:rep-costly-deferral} tests both comparisons on the replenishment model.

\paragraph{Encoding and comparison scope.}
The restriction can be implemented through a mask or a stage-dependent decoder and does not prescribe a policy-output representation. A continuous selector decoded against a local route list can assign the same output coordinate to different physical routes across interfaces; such comparisons are not fixed-global-logit mask ablations. Categorical implementations may also differ in output dimensions and termination aliases. Matching observations and hidden-layer widths therefore does not by itself equalize action distributions, parameter counts, or effective entropy regularization. These implementation properties are distinct from the semantic eligibility rule.

\subsection{Action Abstraction and State-Dependent Execution}
\label{app:action-execution-details}

Action abstraction separates the choice of a semantic decision from its parameters. The policy selects an eligible route or termination control, and the executor $E_n(s,\tilde a)$ produces the concrete action. Route eligibility is determined by the interface before execution; the executor does not make an excluded route eligible.

\paragraph{Forecast-based amount rule.}
For a synthetic cash route $\tilde a=(i,j)$, let $b_i$ denote the source balance and $f_i$ its currently visible outflow forecast. Define source surplus $u_i=(b_i-f_i)^+$ and recipient shortfall $g_j=(f_j-b_j)^+$. The executor protects the source's forecast requirement and allocates surplus to a recipient shortfall or investment:
\begin{equation}
\bar q_{ij}=\begin{cases}
u_i, & j\text{ is the investment account},\\
\min\{u_i,g_j\}, & \text{otherwise},
\end{cases}
\qquad q_{ij}=\min\{\bar q_{ij},b_i^+\}.
\label{eq:section4-appendix-amount}
\end{equation}
AA supplies this amount directly. Basic route--amount policies additionally select a ratio $\rho\in[0,1]$ multiplying the route-specific amount before balance clipping. Thus, AA removes the remaining learned ratio choice rather than introducing forecast-based amount semantics for the first time.

\paragraph{Concrete execution and approximation.}
Writing $(i,j,q)$ for a transfer after numerical clipping, the executor takes the form
\begin{equation}
E_n(s,\tilde a)=\begin{cases}
(i,j,q_{ij}), & \tilde a=(i,j),\\
a_n^{\mathrm{end}}, & \tilde a=\mathrm{STOP}.
\end{cases}
\label{eq:section4-appendix-executor}
\end{equation}
For a given route and state, AA supplies one amount rather than optimizing over all feasible amounts. Its image can therefore omit an optimal concrete action even when the semantic deferral restriction is lossless. Comparisons between AA and route--amount policies evaluate their complete implementations, not an isolated executor effect.

\subsection{Policy-Guided Search: Values and Execution}
\label{app:policy-guided-search-details}

Search uses a frozen policy to guide candidate construction and continuation rollouts. It evaluates alternative current actions, executes only the selected first action, and replans after receiving a new observation. Policy parameters and the executor remain fixed; Search is not a policy-training update.

\paragraph{Rollout action values.}
For weighted hypothetical worlds with $w_k\geq0$ and $\sum_k w_k=1$, let $r_t^{(k)}$ be the reward at simulated transition $t$, with $t=0$ denoting the candidate action, and $T_k$ the number of transitions to termination. The evaluated synthetic controller uses
\begin{equation}
\widehat Q^{\mathrm{roll}}_\theta(s,n,\tilde a)=\sum_k w_k\sum_{t=0}^{T_k-1}r_t^{(k)}.
\label{eq:search-rollout-action-value}
\end{equation}
AA candidates use the amount executor; non-AA candidates include an amount ratio. Subsequent actions follow the frozen policy. The score includes step rewards and terminal settlement without discounting or a learned terminal-value bootstrap. It is a planning-model estimate under policy continuation, not the optimal $Q^\star$ in the theorem.

\paragraph{Information and computation.}
Planning uses current observations and public model information, not the evaluation trajectory's unrevealed outcomes. Complete policy rollouts follow the rollout approach to sequential decision making \citep{bertsekas1997rollout} and account for delayed consequences without fitting a tail-value estimator.

A per-decision transition cap bounds planning work, separately from the candidate-count limit. Candidates admitted to comparison must receive complete scores. Model and scenario approximations remain even with complete rollouts, so a larger cap need not improve realized return and a common cap need not imply equal realized computation. The interface itself does not require this particular score estimator or planner.

\section{Controlled Replenishment}
\label{app:replenishment-controlled-evaluation}

\subsection{Environment and Experimental Protocol}
\label{app:replenishment-environment-protocol}

\paragraph{Ordering and inventory dynamics.}
Each episode contains $W$ weekly ordering cycles followed by a final delivery week. In each cycle, the early decision occurs on Monday and the late decision on Thursday, after Monday--Wednesday demand. A regular supplier accepts orders only at the early decision, whereas a flexible supplier remains available until the late decision. Both suppliers deliver that cycle's orders together on the following Monday, before the next early observation. Inventory carries over between days and weeks. Daily demand is served from available inventory, and unmet demand is lost rather than backlogged. The final delivery week is fully settled without further ordering.

Let $c_r$ and $c_f$ denote regular and flexible unit procurement costs, $v$ unit sales revenue, $h$ daily holding cost, and $b$ the unit unmet-demand penalty. For daily demand $d_t$ and inventory $x_t$ available before demand, sales and unmet demand are $s_t=\min(x_t,d_t)$ and $\ell_t=(d_t-x_t)^+$, and end-of-day inventory is $x_t^+=(x_t-d_t)^+$. With total regular and flexible orders $Q_r$ and $Q_f$, episode return is
\begin{equation}
G=\sum_{t=1}^{7(W+1)}\bigl(vs_t-hx_t^+-b\ell_t\bigr)-c_rQ_r-c_fQ_f.
\end{equation}
Procurement is charged when orders are placed, the objective is undiscounted, and terminal inventory has zero salvage value. Initial inventory $I_0$ is fixed independently of realized demand. Flexible orders share capacity $K_f$ within each ordering cycle; the regular-order decoder has upper bound $K_r$.

\paragraph{Demand and forecast information.}
Weekly demand categories $Z_w\in\{1,2,3\}$ are independent and equiprobable, with associated totals $D_w\in\{\bar D/2,\bar D,3\bar D/2\}$. Conditional on $D_w$, daily demand follows a multinomial allocation across the seven days with equal probabilities. A fine signal $F_w$ identifies the category with probability $p$, and a coarse signal $C_w$ is a symmetric degradation of $F_w$:
\begin{equation}
\begin{aligned}
\Pr(F_w=f\mid Z_w=z)&=\begin{cases}p,&f=z,\\(1-p)/2,&f\ne z,\end{cases}\\[3pt]
\Pr(C_w=c\mid F_w=f)&=\begin{cases}q,&c=f,\\(1-q)/2,&c\ne f.\end{cases}
\end{aligned}
\label{eq:rep-signal-channels}
\end{equation}
The signal model therefore factors as $Z_w\rightarrow F_w\rightarrow C_w$. To vary fine-signal accuracy while keeping coarse accuracy $r$ fixed, we use
\begin{equation}
q=\frac{r-(1-p)/2}{p-(1-p)/2}.
\label{eq:rep-coarse-preservation}
\end{equation}
This preserves the complete conditional distribution of the coarse signal given the demand category. Accuracy denotes correct classification of the weekly demand category, not a percentage error in a continuous forecast.

At the early decision in week $w$, both interfaces observe $C_w$ and $C_{w+1}$. Updated additionally reveals $F_{w+1}$ at the late decision, targeting the same delivery-week demand as $C_{w+1}$ and retaining the coarse signal. The released fine signal remains visible when the delivery week becomes the current week. Retained withholds fine signals; intervening demand and inventory observations remain available in both conditions. Neither interface observes latent demand categories or future realized demand.

\paragraph{Decision interfaces and observations.}
Flat permits flexible orders at both decisions, whereas MMDP defers them to the late decision. Both use the same three-dimensional action box and continuous quantity decoder. For $a\in[-1,1]^3$, define $u_j=\operatorname{clip}((a_j+1)/2,0,1)$. The early regular order is $\bar D(u_0+u_1)$, with $K_r=2\bar D$; regular ordering is disabled at the late decision. Requested flexible quantity is $K_fu_2$, clipped to remaining weekly capacity, and is suppressed at the early decision under MMDP. Quantities are not rounded to integers. All four interface--information combinations use the common observation fields in Table~\ref{tab:rep-observation-fields}, including explicit remaining flexible capacity. Unreleased fine-signal slots and visibility indicators are zero.

Postponing a fixed feasible flexible order preserves its price, delivery date, total capacity consumption, and subsequent physical availability. Since neither early nor late orders can serve demand before the next Monday, postponement changes only the timing of the procurement charge within an undiscounted objective. This establishes the benchmark's lossless local deferral property, not a finite-budget learning guarantee.

\paragraph{Environment parameters.}
Table~\ref{tab:rep-environment-parameters} lists the numerical settings. Inventory and capacity are expressed relative to mean weekly demand, and monetary quantities relative to regular unit procurement cost. These are synthetic benchmark parameters shared across interfaces and information conditions, rather than a calibration to a particular business.

\begin{table}[htbp]
\centering
\small
\setlength{\tabcolsep}{4pt}
\renewcommand{\arraystretch}{1.08}
\caption{Replenishment environment parameters. The main-text setting uses $p=0.80$; the corresponding degradation probability $q$ follows Eq.~\eqref{eq:rep-coarse-preservation}.}
\label{tab:rep-environment-parameters}
\begin{tabular*}{\linewidth}{@{\extracolsep{\fill}}lll@{}}
\toprule
Parameter & Symbol & Value \\
\midrule
Ordering cycles & $W$ & $13$ ($14$ demand weeks) \\
Mean weekly demand & $\bar D$ & $84$ units \\
Weekly demand totals & $D_w$ & $42,84,126$, equiprobable \\
Initial inventory & $I_0$ & $0.75\bar D=63$ \\
Regular-order upper bound & $K_r$ & $2\bar D=168$ \\
Shared flexible capacity & $K_f$ & $0.8\bar D=67.2$ \\
Regular / flexible unit costs & $c_r,c_f$ & $1.0,1.5$ \\
Unit sales revenue & $v$ & $6.0$ \\
Daily unit holding cost & $h$ & $0.5$ \\
Unit unmet-demand penalty & $b$ & $3.0$ \\
Coarse classification accuracy & $r$ & $0.57875$ \\
Fine classification accuracy & $p$ & $0.80$ (main text) \\
\bottomrule
\end{tabular*}
\end{table}

\begin{table}[htbp]
\centering
\small
\setlength{\tabcolsep}{3pt}
\renewcommand{\arraystretch}{1.08}
\caption{Common 24-dimensional observation. Indices are zero-based; signal categories use three-dimensional one-hot encodings.}
\label{tab:rep-observation-fields}
\begin{tabular}{@{}lp{0.59\linewidth}l@{}}
\toprule
Indices & Field & Encoding / scale \\
\midrule
0--1 & Early/late stage & One-hot \\
2 & Remaining ordering cycles & $/W$ \\
3 & Current inventory & $/K_r$ \\
4 & Regular orders arriving next Monday & $/K_r$ \\
5 & Early flexible orders placed this week & $/K_f$ \\
6 & Remaining flexible capacity & $/K_f$ \\
7 & Inventory at this week's Monday opening & $/K_r$ \\
8 & Observed demand accumulated this week & $/(1.5\bar D)$ \\
9--11 & Current-week coarse signal & One-hot \\
12--14 & Current-week fine signal, if released & One-hot or zero \\
15 & Current-week fine-signal visibility & $0/1$ \\
16--18 & Delivery-week coarse signal & One-hot \\
19--21 & Delivery-week fine signal, if released & One-hot or zero \\
22 & Delivery-week fine-signal visibility & $0/1$ \\
23 & Updated-condition indicator & $0/1$ \\
\bottomrule
\end{tabular}
\end{table}

\paragraph{Matched training and policy selection.}
Each setting crosses Flat and MMDP with Retained and Updated, using $N$ paired training seeds. All policies are trained from scratch with identical observation dimensions and actor/critic architectures. Within each seed, the four cells share initial weights and paired exogenous episode streams; every training reset generates a fresh episode. PPO samples Gaussian actions and clips them to the common action box. Training rewards are scaled by a constant $\kappa$, while all reported returns use original economic units. Table~\ref{tab:rep-training-parameters} lists the architecture, optimizer, and evaluation settings.

Checkpoints are evaluated at regular transition intervals on a shared validation bank. Each run selects the highest validation mean, with ties resolved in favor of the earliest checkpoint. Initialization evaluations and the checkpoint after completion of the final rollout batch are excluded from the selection candidates. Selected policies are evaluated deterministically on a shared test bank after all selections in that setting have been fixed. The test instances are excluded from training and per-run checkpoint selection.

\paragraph{Matched information conditions.}
Training, validation, and test generators use distinct split keys; training streams also depend on training seed and reset index. Demand, signals, and daily allocations use separate random-number streams. Revealing or withholding a fine signal does not change exogenous demand. Across fine-accuracy settings, demand and coarse-signal realizations remain fixed, and fine signals are sampled from $\Pr_p(F\mid Z,C)$ under the required joint model. Latent variables used for this coupling are never policy inputs.

\paragraph{Metrics and uncertainty.}
Each validation point averages the validation instances within a seed; learning curves show the across-seed mean and one sample standard deviation. For validation return $R_i(t_k)$ at checkpoint $k=1,\ldots,K$, normalized validation AUC is
\begin{equation}
\mathrm{AUC}_i=\frac{1}{t_K-t_1}\sum_{k=1}^{K-1}\frac{R_i(t_k)+R_i(t_{k+1})}{2}(t_{k+1}-t_k).
\label{eq:replenishment-normalized-auc}
\end{equation}
It covers only the observed training interval and is expressed in return units. Test return first averages all test instances within a selected policy, yielding one value per training seed, rather than treating individual trajectories as independent training replicates.

Let $Y_{i,m,j}$ denote either seed-level metric for interface $m\in\{\mathrm{Flat},\mathrm{MMDP}\}$ and information condition $j\in\{\mathrm{R},\mathrm{U}\}$. In addition to interface differences and update gains, we calculate the paired interaction
\begin{equation}
I_i=\bigl(Y_{i,\mathrm{MMDP},\mathrm{U}}-Y_{i,\mathrm{MMDP},\mathrm{R}}\bigr)-\bigl(Y_{i,\mathrm{Flat},\mathrm{U}}-Y_{i,\mathrm{Flat},\mathrm{R}}\bigr).
\label{eq:replenishment-refinement-interaction}
\end{equation}
Percentile-bootstrap intervals resample complete seed rows, preserving the pairing of all four cells. The reported 95\% intervals quantify training-seed variability conditional on the evaluation bank; they are marginal intervals without adjustment for multiple comparisons or setting selection.

\begin{table}[htbp]
\centering
\small
\setlength{\tabcolsep}{4pt}
\renewcommand{\arraystretch}{1.08}
\caption{Base PPO and evaluation settings, shared across the four interface--information combinations.}
\label{tab:rep-training-parameters}
\begin{tabular*}{\linewidth}{@{\extracolsep{\fill}}ll@{}}
\toprule
Parameter & Value \\
\midrule
Paired training seeds $N$ & $10$ (42--51) \\
Actor / critic hidden layers & $[128,128]$, Tanh \\
Trainable parameters per model & $39{,}943$ \\
Learning rate & $3\times10^{-4}$ \\
Rollout steps / minibatch size / epochs & $256$ / $64$ / $10$ \\
Discount / GAE parameter & $1.0$ / $0.95$ \\
PPO clipping parameter & $0.2$ \\
Entropy / value coefficients & $0$ / $0.5$ \\
Maximum global gradient norm & $0.5$ \\
Training reward scale $\kappa$ & $0.01$ \\
Nominal / actual training transitions & $100{,}000$ / $100{,}096$ \\
Validation checkpoints & $5{,}000,10{,}000,\ldots,100{,}000$ \\
Validation / test instances & $20$ / $200$ \\
Normalized AUC interval & $[5{,}000,100{,}000]$ \\
Bootstrap resamples / seed & $100{,}000$ / $9245000$ (PCG64) \\
Implementation & Stable-Baselines3 2.7.1, PyTorch 2.6.0 \\
\bottomrule
\end{tabular*}
\end{table}

\subsection{Two-Stage and Daily-Review Replenishment}
\label{app:rep-multi-stage}

\paragraph{Two-stage learning and test performance.}
MMDP achieves higher mean validation AUC and test return under both information conditions. Table~\ref{tab:rep-daily-test} reports the two-stage benchmark with forecast updating alongside the daily-review results. Appendix~\ref{app:replenishment-results-analysis} gives the full two-stage results for both information conditions.

\paragraph{Daily-review environment and decision interfaces.}
We replace the Monday and Thursday ordering reviews with one review per day. The regular and flexible suppliers retain their weekly capacities $K_r$ and $K_f$ and deliver all orders the following Monday. We divide $K_r$ equally among quotas expiring on Monday, Tuesday, and Wednesday, preserving the regular supplier's cutoff before Thursday. For the flexible supplier, one eighth of $K_f$ expires on each of the first six days, and the remaining quarter expires on Sunday. Flat can use any unexpired quota and draws from the earliest deadline first. The stage-local interface can use only quota expiring that day. Each episode retains the same 13 ordering weeks and final delivery week as the two-stage benchmark.

\paragraph{Environment parameters.}
Table~\ref{tab:rep-daily-parameters} lists the daily schedule. All seven forecasts concern the same delivery-week demand category. We sample intermediate signals from the degradation chain conditional on the original coarse and fine endpoint draws. The observation has seven stage indicators, ten progress, inventory, and quota features, and 42 one-hot forecast entries for the current and delivery weeks. Unrevealed entries are zero. The two review schedules share demand trajectories.

\begin{table}[htbp]
\centering
\small
\caption{Daily-review parameters. Quotas are listed in expiry order from Monday through Sunday. The weekly capacities are $K_r=168$ and $K_f=67.2$ units.}
\label{tab:rep-daily-parameters}
\begin{tabular*}{\linewidth}{@{\extracolsep{\fill}}ll@{}}
\toprule
Parameter & Setting \\
\midrule
Regular quota (units) & $(56,56,56,0,0,0,0)$ \\
Flexible quota (units) & $(8.4,8.4,8.4,8.4,8.4,8.4,16.8)$ \\
Forecast accuracy & Linear from $r=57.875\%$ on Monday to $p=80\%$ on Sunday \\
\bottomrule
\end{tabular*}
\end{table}

\paragraph{Training and evaluation.}
Both interfaces use the PPO settings in Table~\ref{tab:rep-training-parameters}, with a wider input layer for the daily observation and an entropy coefficient of $0.01$. We train from scratch with ten paired seeds at nominal budgets of $100{,}000$ and $250{,}000$ transitions, yielding $100{,}096$ and $250{,}112$ actual transitions after completing the final rollout. Every $5{,}000$ nominal transitions, we evaluate on 20 shared validation instances. We select each policy's checkpoint by mean validation return and evaluate on the same 200 held-out instances. The longer runs reproduce the short runs' validation curves through $100{,}000$ transitions for all twenty policies.

\paragraph{Results.}
At $50{,}000$ transitions, the mean validation gap is $10{,}539.0$ with daily review and $168.1$ with two-stage review. The daily-review gap is $311.7$ at $250{,}000$ transitions. Table~\ref{tab:rep-daily-test} reports held-out performance. MMDP wins eight of ten seed pairs at the shorter daily-review budget and nine at the longer budget. One Flat seed returns $-15{,}998.8$ at the shorter budget, contributing to the large mean gap and dispersion. All seeds are included. The schedules differ in review frequency, observations, quota deadlines, entropy coefficients, and final training budgets.

\begin{table}[htbp]
\centering
\small
\caption{Replenishment test returns under two-stage review with forecast updating and daily review. Policy columns give the mean and sample SD over ten training-seed means. Contrasts use paired-bootstrap 95\% CIs over training seeds.}
\label{tab:rep-daily-test}
\begin{tabular*}{\linewidth}{@{\extracolsep{\fill}}lrrrc@{}}
\toprule
Review schedule & Actual steps & Flat (SD) & MMDP (SD) & MMDP $-$ Flat [95\% CI] \\
\midrule
Two-stage & $100{,}096$ & $2381.5\ (147.6)$ & $2514.0\ (70.1)$ & $+132.5\ [56.6,213.0]$ \\
Daily & $100{,}096$ & $-838.4\ (5471.8)$ & $1655.6\ (445.1)$ & $+2493.9\ [103.8,6259.6]$ \\
Daily & $250{,}112$ & $1867.0\ (260.8)$ & $2081.5\ (112.7)$ & $+214.6\ [95.8,350.6]$ \\
\bottomrule
\end{tabular*}
\end{table}

\subsection{Value of Refined Information}
\label{app:replenishment-results-analysis}

\paragraph{Forecast updating in the two-stage benchmark.}
Tables~\ref{tab:rep-test-cell-means} and~\ref{tab:replenishment-complete-results} compare policies trained with and without the refined forecast. Forecast updating improves both policy families' mean test returns, with a larger improvement for MMDP at the main setting. The test interaction is $68.4$ (95\% CI $[3.0,136.2]$), while the AUC interaction remains unresolved. These are comparisons between separately trained and validation-selected policies, not responses of a single frozen policy to changed signals.

\begin{table}[htbp]
\centering
\small
\renewcommand{\arraystretch}{1.08}
\caption{Mean and sample SD are calculated over ten training-seed means, each averaging 200 shared instances.}
\label{tab:rep-test-cell-means}
\begin{tabular*}{\linewidth}{@{\extracolsep{\fill}}llrr@{}}
\toprule
Interface & Information & Mean return & Seed SD \\
\midrule
Flat & Retained & $2095.4$ & $64.1$ \\
Flat & Updated & $2381.5$ & $147.6$ \\
MMDP & Retained & $2159.5$ & $38.5$ \\
MMDP & Updated & $2514.0$ & $70.1$ \\
\bottomrule
\end{tabular*}
\end{table}

\begin{table}[htbp]
\centering
\small
\setlength{\tabcolsep}{3pt}
\renewcommand{\arraystretch}{1.12}
\caption{Paired contrasts with marginal paired-bootstrap 95\% CIs. Interaction is defined in Eq.~\eqref{eq:replenishment-refinement-interaction}.}
\label{tab:replenishment-complete-results}
\begin{tabular*}{\linewidth}{@{\extracolsep{\fill}}lcc@{}}
\toprule
Contrast & Validation AUC & Test return \\
\midrule
MMDP $-$ Flat, Retained & $+170.0\ [113.1,219.2]$ & $+64.1\ [20.2,107.1]$ \\
MMDP $-$ Flat, Updated & $+180.3\ [112.1,247.8]$ & $+132.5\ [56.6,213.0]$ \\
Updated $-$ Retained, Flat & $+244.0\ [184.1,301.1]$ & $+286.1\ [218.3,347.7]$ \\
Updated $-$ Retained, MMDP & $+254.4\ [221.1,281.7]$ & $+354.5\ [298.5,409.8]$ \\
Interaction & $+10.3\ [-34.8,60.4]$ & $+68.4\ [3.0,136.2]$ \\
\bottomrule
\end{tabular*}
\end{table}

\paragraph{Forecast-accuracy sensitivity.}
Table~\ref{tab:rep-accuracy-test} varies fine accuracy with the coarse-signal law, demand realizations, and base PPO configuration held fixed. We selected $80\%$ for the main text after exploring accuracy levels and PPO configurations on the same evaluation bank. Both interfaces' Updated mean test returns increase with fine accuracy, and MMDP has the higher mean at all three levels. The test interaction estimates are positive at all three levels, with a marginal 95\% interval excluding zero at $80\%$.

\begin{table}[htbp]
\centering
\small
\setlength{\tabcolsep}{3pt}
\renewcommand{\arraystretch}{1.08}
\caption{Test sensitivity to fine-signal accuracy under base PPO. Flat and MMDP columns are Updated mean returns; interaction includes its marginal 95\% CI. Retained means are unchanged at 2095.4 (Flat) and 2159.5 (MMDP).}
\label{tab:rep-accuracy-test}
\begin{tabular*}{\linewidth}{@{\extracolsep{\fill}}lrrc@{}}
\toprule
Fine accuracy & Flat & MMDP & Test interaction [95\% CI] \\
\midrule
$75\%$ & $2260.9$ & $2375.2$ & $+50.2\ [-35.7,154.5]$ \\
$80\%$ & $2381.5$ & $2514.0$ & $+68.4\ [3.0,136.2]$ \\
$85\%$ & $2502.6$ & $2625.2$ & $+58.6\ [-39.4,164.0]$ \\
\bottomrule
\end{tabular*}
\end{table}

\subsection{Late-Order Surcharge}
\label{app:rep-costly-deferral}

\paragraph{Environment and decision interfaces.}
We keep the two-stage review schedule and vary the cost of postponing a flexible order. A late flexible order pays a per-unit surcharge $\lambda$ on top of the base cost $c_f$. We tighten regular capacity to increase the role of flexible supply. The same cap applies at every surcharge, including the zero-surcharge control. Flat can place a flexible order at either review. EO defers flexible orders to the late review, while the MMDP policy can retain early flexible orders through $\widehat{\mathcal E}_n(s)$.

\paragraph{Environment parameters.}
The regular capacity is $K_r=60$ and the shared flexible capacity is $K_f=67.2$ units. The flexible unit cost is $c_f=1.5$ at the early review and $c_f+\lambda$ at the late review. We use the original demand generator and the Updated forecast protocol with $p=0.80$, and retain daily holding cost $h=0.5$. Table~\ref{tab:rep-costly-deferral} lists all surcharge levels for this configuration.

\paragraph{Training and evaluation.}
At each surcharge, Flat and EO train from scratch with the base PPO settings in Table~\ref{tab:rep-training-parameters} for $100{,}096$ transitions over ten paired seeds. We select checkpoints by mean return on 20 shared validation instances and evaluate them on 200 shared held-out instances.

\paragraph{Exact reference on matched terminal-week states.}
For the exact comparison, we use the final ordering cycle, which retains both reviews and the forecast update. After the late order, only delivery and settlement remain, so we can compute the optimal value by integrating over conditional forecast and demand outcomes and optimizing order quantities at the breakpoints of the piecewise-linear value functions. We average exact EO minus Flat values over 12 equally weighted states formed by opening inventories $\{0,60,100,130\}$ and the three coarse delivery forecasts. We fix both current-week forecast signals at the middle category. The same states are used to construct and check the model-based action sets and to train and evaluate the PPO policies in Figure~\ref{fig:rep_terminal_ppo}.

\paragraph{Results.}
Table~\ref{tab:rep-costly-deferral} reports exact terminal-week gaps and PPO returns over full episodes. The first nonzero exact loss on the evaluated grid occurs at $\lambda=1$, while the mean PPO gap first turns negative at $\lambda=2.5$. The exact reference uses a common state grid, and the PPO policies generate their own inventory trajectories across the full episode. At $\lambda=1$, EO gains $356.3$ before surcharges and pays $146.4$ more in surcharges, yielding the net gap of $209.9$. At $\lambda=3$, the corresponding amounts are $155.9$ and $247.1$, yielding a net gap of $-91.2$.

\begin{table}[htbp]
\centering
\small
\caption{Late-order surcharge results. Exact gaps average matched terminal-week states. PPO columns give full-episode means and sample SDs across ten training seeds. Contrasts use paired-bootstrap 95\% CIs.}
\label{tab:rep-costly-deferral}
\setlength{\tabcolsep}{4pt}
\begin{tabular*}{\linewidth}{@{\extracolsep{\fill}}rrrrc@{}}
\toprule
& & \multicolumn{3}{c}{Full-episode PPO test return} \\
\cmidrule(l){3-5}
$\lambda$ & Exact gap & Flat (SD) & EO (SD) & EO $-$ Flat [95\% CI] \\
\midrule
$0$ & $0.0$ & $2463.6\ (89.6)$ & $2506.6\ (54.8)$ & $+43.0\ [-21.1,98.9]$ \\
$0.1$ & $0.0$ & $2420.5\ (79.9)$ & $2495.0\ (50.4)$ & $+74.6\ [12.7,132.6]$ \\
$0.5$ & $0.0$ & $2191.3\ (91.5)$ & $2370.2\ (59.6)$ & $+179.0\ [127.8,237.8]$ \\
$1$ & $-2.6$ & $2043.5\ (70.7)$ & $2253.5\ (32.3)$ & $+209.9\ [160.4,258.7]$ \\
$1.5$ & $-6.4$ & $1976.8\ (81.1)$ & $2131.6\ (44.6)$ & $+154.8\ [98.6,203.2]$ \\
$2$ & $-10.8$ & $1961.9\ (63.1)$ & $2013.9\ (35.1)$ & $+52.0\ [14.7,88.7]$ \\
$2.5$ & $-16.6$ & $1939.5\ (62.9)$ & $1919.6\ (20.5)$ & $-19.9\ [-58.8,22.4]$ \\
$3$ & $-22.6$ & $1945.2\ (61.7)$ & $1854.1\ (20.2)$ & $-91.2\ [-120.5,-57.6]$ \\
\bottomrule
\end{tabular*}
\end{table}

\paragraph{Model-based screening on terminal-week states.}
We apply Eq.~\eqref{eq:action-wise-policy} to the same 12 matched states at the seven surcharge levels in Figure~\ref{fig:rep_terminal_ppo}. We compute a feasible expiring alternative by optimizing a regular-only early order and a late flexible top-up that responds to observed current-week demand and the public delivery-week forecast. For each early order candidate, we hold its quantity and cost fixed and give the late top-up access to both latent weekly demand totals. Optimizing under this information relaxation gives an upper bound on the continuation value. Capacity, procurement prices, and arrival timing stay the same. We report the fraction screened among the candidate breakpoints with positive early flexible quantity. We then solve the unrestricted and screened problems exactly under public information, choosing the smallest early flexible quantity among tied maximizers.

To test when early flexible ordering is necessary, we evaluate four feasible early orders that use the regular quota and commit $25\%$, $50\%$, $75\%$, or $100\%$ of flexible capacity, with a public-signal top-up. We compare their best value with two upper bounds on plans without early flexible orders. Both relax information only at the late review. The first reveals both demand totals. The second reveals only the current-week total, which determines inventory carryover, and keeps the public forecast for delivery-week demand. A feasible early flexible order that beats either bound certifies that an optimal plan must order some flexible units early.

\begin{table}[htbp]
\centering
\small
\caption{Terminal-week screening by late-order surcharge. Certification columns indicate which demand totals are revealed to the late decision. The exact count gives the number of states where early flexible ordering is necessary. Counts are out of 12 states per surcharge. The screened and unrestricted exact optima coincide across all 84 cases.}
\label{tab:rep-screening-audit}
\begin{tabular*}{\linewidth}{@{\extracolsep{\fill}}rcrrr@{}}
\toprule
& & \multicolumn{2}{c}{Certified early ordering} & \\
\cmidrule(lr){3-4}
$\lambda$ & Candidates removed (\%) & Both weeks & Current week & Exact count \\
\midrule
$0$ & $79.1$ & $0$ & $0$ & $0$ \\
$0.5$ & $75.3$ & $0$ & $0$ & $0$ \\
$1$ & $70.4$ & $0$ & $4$ & $6$ \\
$1.5$ & $64.9$ & $0$ & $4$ & $6$ \\
$2$ & $61.9$ & $0$ & $5$ & $6$ \\
$2.5$ & $59.9$ & $0$ & $5$ & $7$ \\
$3$ & $58.5$ & $3$ & $6$ & $7$ \\
\bottomrule
\end{tabular*}
\end{table}

At $\lambda=1$, with 60 units in opening inventory and a high coarse forecast for delivery-week demand, ordering 60 regular units and 16.8 flexible units early and topping up using public information gives an expected value of $219.70$. The best plan without an early flexible order has an upper bound of $214.60$ when the current-week demand total is revealed at the late review. Thus early flexible ordering is necessary in this state. Table~\ref{tab:rep-screening-audit} shows where the four feasible probes and the two bounds certify early ordering, alongside the exact count of states in which the fixed interface loses value. The full-episode PPO comparisons use the fixed interface in Table~\ref{tab:rep-costly-deferral}.

\paragraph{Crossover in the exact terminal-week model.}
Write $q$ for the early order, $C(q)$ for its procurement cost, and $t$ for the total order after the late top-up. The regular quota is used before early flexible capacity. If $G_s(z,t)$ is the expected terminal payoff before procurement conditional on the late public signal $z$, the exact action value is
\begin{equation}
Q_s(q;\lambda)=-C(q)+\mathbb E_{z\mid s}
\left[\max_{q\le t\le \min\{q+K_f,K_r+K_f\}}
  \bigl\{G_s(z,t)-(c_f+\lambda)(t-q)\bigr\}\right].
\end{equation}
For the crossover calculation, the early order $q$ ranges over the finite set of order-quantity breakpoints. The screened set preserves the Flat optimum, so it strictly improves on EO exactly when $\max_{q>K_r}Q_s(q;\lambda)>\max_{q\le K_r}Q_s(q;\lambda)$. Concavity of $G_s(z,\cdot)$ lets us clip a maximizer of $G_s(z,t)-(c_f+\lambda)t$ to the feasible interval for each $q$. Thus, for any fixed $q>K_r$ and $q'\le K_r$, the former uses weakly fewer late flexible units in each signal. By the envelope theorem, its value advantage over $q'$ is nondecreasing in $\lambda$. Taking the minimum of this difference over $q'$ and then the maximum over $q>K_r$ preserves monotonicity. Solving for the first positive difference on the matched grid gives a crossover of approximately $0.539$ in the zero-inventory states with a medium or high delivery-week forecast. Six of the 12 states strictly favor early orders at $\lambda=1$. The first state-level crossover is also where the aggregate exact advantage begins. The finite-budget PPO comparison is evaluated at the seven specified surcharges.

\paragraph{Terminal-week PPO.}
At $\lambda\in\{0,0.5,1,1.5,2,2.5,3\}$, Flat, EO, and MMDP share a discrete action head over the same quantity breakpoints. At each early review, they mask this head to $\mathcal A_n$, $\mathcal D_n$, or precomputed $\widehat{\mathcal E}_n(s)$, respectively. All three use the same feasible late-order grid. At $\lambda=0$, the experiment uses $\widehat{\mathcal E}_n(s)$ to isolate the second branch of Eq.~\eqref{eq:stage-local-policy}, whose costless-deferral branch would use $\mathcal D_n$. We train masked PPO for $65{,}536$ transitions per policy over ten paired initializations, with two hidden layers of 64 units, a learning rate of $10^{-3}$, and $\gamma=1$. Each update uses $2{,}048$ transitions from 16 parallel environments, ten epochs, and minibatches of 256. The entropy coefficient is $0.005$. Training samples public signals in the 12 matched states and uses the conditional expected terminal payoff, scaled by $1/100$. We evaluate final greedy policies by exact expectation over the public-signal distribution on the same state grid.

\begin{table}[htbp]
\centering
\small
\caption{Terminal-week PPO returns in Figure~\ref{fig:rep_terminal_ppo}. Contrasts give paired-bootstrap 95\% CIs over training seeds.}
\label{tab:rep-terminal-ppo}
\setlength{\tabcolsep}{3pt}
\begin{tabular*}{\linewidth}{@{\extracolsep{\fill}}rrrrrcc@{}}
\toprule
$\lambda$ & Exact Flat & Flat & EO & MMDP & MMDP $-$ Flat & MMDP $-$ EO \\
\midrule
$0$ & $226.02$ & $210.64$ & $213.31$ & $211.38$ & $+0.74\,[{-0.24},1.76]$ & $-1.93\,[{-2.91},{-0.90}]$ \\
$0.5$ & $212.81$ & $197.89$ & $199.90$ & $199.21$ & $+1.31\,[0.09,2.50]$ & $-0.69\,[{-1.82},0.29]$ \\
$1$ & $202.77$ & $188.34$ & $188.55$ & $189.86$ & $+1.52\,[0.77,2.22]$ & $+1.31\,[0.47,2.13]$ \\
$1.5$ & $194.04$ & $180.78$ & $178.37$ & $181.51$ & $+0.74\,[{-0.03},1.49]$ & $+3.14\,[1.89,4.41]$ \\
$2$ & $186.34$ & $175.46$ & $167.83$ & $176.64$ & $+1.18\,[0.43,1.97]$ & $+8.80\,[8.21,9.42]$ \\
$2.5$ & $180.16$ & $171.58$ & $157.43$ & $172.84$ & $+1.26\,[0.67,1.92]$ & $+15.41\,[14.56,16.26]$ \\
$3$ & $176.13$ & $168.88$ & $148.24$ & $170.52$ & $+1.64\,[0.49,2.70]$ & $+22.28\,[21.14,23.31]$ \\
\bottomrule
\end{tabular*}
\end{table}

\section{Synthetic Cash Management}

The five-account comparisons examine route restrictions and equal-size controls; the ten-account comparisons examine action abstraction and policy-guided planning. We separate environment specifications, shared methods, benchmark-specific evaluation protocols, and planning-budget accounting.

\subsection{Environment}
\label{app:cash-environment-v7}

\paragraph{Accounts, routes, and daily stages.}
Each synthetic environment contains a master account $M$, an investment account $I$, and a set $\mathcal O$ of operating accounts with uncertain end-of-day cash demand. The five- and ten-account settings have $|\mathcal O|=3$ and $8$, respectively. One episode represents one daily decision cycle with two stages. The full early-stage route catalogue contains both directions between $M$ and $I$, both directions between $M$ and each operating account, and every directed connection between distinct operating accounts. At the stage boundary, the master--operating routes expire. The master--investment routes and inter-operating routes remain available. The full route catalogue therefore contains 14 routes in the five-account setting and 74 in the ten-account setting; respectively, six and sixteen expire, leaving eight and fifty-eight persistent routes.

Flat exposes the complete early-stage route catalogue, whereas the MMDP interface exposes only its expiring routes at that stage. Both expose the same persistent routes after the boundary, with termination available in both stages. Selecting termination ends the current stage without making another transfer. Reaching the stage's operation limit also advances the environment. The limits are ten non-termination operations per stage in the five-account setting and sixteen in the ten-account setting. A zero-amount transfer still counts as an operation. These are decision-operation limits, not subdivisions of physical hours.

\paragraph{Demand and forecast generation.}
At reset, the generator samples the day's nonnegative base demand, a possible localized demand shock, and the forecasts revealed during the episode. Base demand for operating account $i$ is $d_i^{\mathrm{base}}=\max\{0,\mu_i+\sigma_i\epsilon_i\}$, where $\epsilon_i\sim\mathcal N(0,1)$. The initial forecast adds independent zero-mean Gaussian noise with standard deviation $\nu\max\{d_i^{\mathrm{base}},1\}$ and is clipped at zero. With probability $0.9$, a single operating account receives an additional nonnegative shock; otherwise the shock vector is zero. The master and investment accounts have no exogenous demand. Table~\ref{tab:cash-environment-parameters} lists the configuration parameters, with operating-account vectors ordered by account index.

Stage zero reveals only the initial forecast; the realized shock and updated forecast remain hidden. At the boundary, the updated forecast and shock information are revealed simultaneously with expiration of the early-only routes. The evaluated configurations use zero second-stage reveal noise, so the updated forecast equals the total end-of-day demand. This is a controlled full-revelation endpoint, not a claim that operational forecasts are perfect. Pre-sampling the full episode makes paired evaluation reproducible but does not grant the controller access to unrevealed outcomes.

\paragraph{Observations and transfers.}
Within each matched comparison, Flat and MMDP receive the same stage indicator, within-stage progress, account balances, current demand forecast, forecast liquidity gaps and surpluses, and revealed-shock fields. The continuous PPO parameterization selects a local route and an amount ratio $\rho\in[0,1]$. Let $b_i$ be the current balance and $\widehat d_i$ the current forecast. Define forecast surplus $u_i=(b_i-\widehat d_i)^+$ and forecast gap $g_j=(\widehat d_j-b_j)^+$, where $x^+=\max\{x,0\}$. The route--amount executor computes
\begin{equation}
q_{ij}=\begin{cases}
\rho u_i, & j=I,\\
\rho\min\{u_i,g_j\}, & j\ne I,
\end{cases}
\qquad q_{ij}\leftarrow\min\{q_{ij},b_i^+\}.
\label{eq:cash-raw-amount-executor}
\end{equation}
The transfer immediately subtracts $q_{ij}$ from the source balance and adds it to the destination balance. The public environments do not introduce separate per-route capacity parameters: executable quantities are constrained by balances and the forecast-based amount rule. Edge-only action-abstraction variants use their separately specified quantity executor; they should not be interpreted as continuous route--amount policies with an identical action parameterization.

\paragraph{Settlement and reward.}
No demand is settled at the intermediate stage boundary. After the final stage, the total daily demand $d_i$ is deducted once from each account. If $b_i^{\mathrm{pre}}$ is the balance immediately before settlement, write $\widetilde b_i=b_i^{\mathrm{pre}}-d_i$. The undiscounted episode return is
\begin{equation}
R=\sum_i y_i\widetilde b_i^+-\sum_i p_i(-\widetilde b_i)^+-f\sum_{t\in\mathcal T}q_t-c|\mathcal T|-c_0\sum_{t\in\mathcal T}\mathbf 1\{q_t\leq10^{-8}\},
\label{eq:cash-episode-return}
\end{equation}
where $\mathcal T$ contains non-termination operations, $y_i$ is the positive-balance yield coefficient, $p_i$ is the shortfall-penalty coefficient, and $f$ is the proportional transfer fee. The per-operation cost is $c$, with an additional zero-transfer penalty $c_0$ in the ten-account environment. Termination has no operation fee, although final termination also returns the settlement reward. Transfer fees are charged in the reward and are not deducted a second time from account balances. Investment-shaping rewards are disabled. All monetary quantities and reward coefficients are synthetic benchmark units rather than calibrated production costs.

\begin{table}[htbp]
\centering
\small
\renewcommand{\arraystretch}{1.12}
\caption{Configurations of the two synthetic Cash environments. Initial balances are sampled uniformly within the listed ranges. Gaussian demand and shock samples are clipped at zero. Account-wise yield and penalty vectors are ordered as master, investment, then operating accounts.}
\label{tab:cash-environment-parameters}
\begin{tabularx}{\linewidth}{@{}l>{\raggedright\arraybackslash}X>{\raggedright\arraybackslash}X@{}}
\toprule
Parameter & Five accounts & Ten accounts \\
\midrule
Initial master balance & $[140,180]$ & $[140,180]$ \\
Initial investment balance & $[20,40]$ & $[30,60]$ \\
Initial operating balances & $[25,45]$, $[20,40]$, $[20,35]$ & $[18,35]$ for each account \\
Base-demand means & $(55,40,35)$ & $(40,38,36,34,32,30,28,26)$ \\
Base-demand standard deviations & $(10,8,7)$ & $(8,8,7,7,6.5,6.5,6,6)$ \\
Initial forecast noise ratio $\nu$ & $0.20$ & $0.18$ \\
Shock probability & $0.90$ & $0.90$ \\
Shock Gaussian mean / SD & $35/12$ & $28/9$ \\
Shock target probabilities & $(.40,.35,.25)$ & $(.16,.15,.14,.13,$ $\ .12,.11,.10,.09)$ \\
Second-stage reveal noise & $0$ & $0$ \\
Yield coefficients & $(.003,.010,.001,.001,.001)$ & $.003$ master, $.010$ investment, $.001$ each operating \\
Shortfall coefficients & $(4,2,10,8,8)$ & $4$ master, $2$ investment, $9$ each operating \\
Transfer fee $f$ & $0.001$ & $0.001$ \\
Operation cost $c$ & $20$ & $10$ \\
Additional zero-transfer cost $c_0$ & $0$ & $2$ \\
\bottomrule
\end{tabularx}
\end{table}

\paragraph{Coupling and comparison scope.}
Expiration is explicit, but a persistent transfer generally cannot be deferred without changing its consequences: an earlier transfer changes balances, the quantities executable on other routes, and the resources available before a cutoff. Cash therefore tests a structured approximate restriction rather than the exact-deferral conditions of the theorem. Comparisons within each controller family hold observations, dynamics, reward, execution rules, and evaluation instances fixed while changing the route interface. The five- and ten-account settings also differ in demand, initial liquidity, costs, and training protocols, so their results are not a controlled estimate of the effect of account count alone.

\subsection{Five-Account Methods and Training}
\label{app:cash5-evaluation}
\label{app:cash-shared-methods}

This subsection specifies the five-account controller implementations, training settings, route-mask construction, and serial timing method.

\subsubsection{Controllers and PPO Training}

\paragraph{Matched interfaces and PPO training.}
Flat and MMDP use the same 28-dimensional observation, two continuous route--ratio outputs, immediate amount executor, and reward function. Flat exposes 14 routes at stage zero and eight at stage one; MMDP exposes the six expiring routes at stage zero and the same eight persistent routes at stage one. Termination is available under both interfaces. Actor and critic each have two 128-unit hidden layers with Tanh activations, giving 40837 policy parameters. PPO uses learning rate $3\times10^{-4}$, rollout length 256, batch size 64, ten optimization epochs, $\gamma=0.99$, GAE parameter $0.95$, clipping parameter $0.2$, entropy coefficient zero, value coefficient $0.5$, and gradient-norm cap $0.5$. The nominal budget is 500000 transitions; completing the final rollout gives an endpoint of 500224 transitions. Every independent test uses the final policy, without validation-based checkpoint selection.

We use ten paired training seeds, 42--51, retaining all runs. The extension from five to ten seeds was specified after earlier results were known; final checkpoints and evaluation rules were fixed before the added test results were accessed.

\paragraph{Random policy.}
Random samples uniformly from the available routes plus termination and independently samples a transfer ratio uniformly on $[0,1]$. Each evaluated environment uses ten paired action-sampling replicates per interface. A separate PCG64 generator is keyed by environment, replicate, stage, and within-stage step, without consuming environment randomness. The stage-zero termination probability is $1/15$ under Flat and $1/7$ under MMDP; both are $1/9$ at stage one. Consequently, the Random comparison includes changes in route identity, action count, and termination probability, rather than isolating expiration semantics alone.

\paragraph{Rolling-horizon MILP controller.}
We evaluate the forecast-based rolling-horizon MILP controller described in the main text. At each decision, it plans transfers from current balances, visible forecasts, public model coefficients, and routes admitted by the interface. Stage-zero planning demand adds the public expected localized shock to the visible forecast; stage-one planning uses the currently revealed forecast. The controller does not access unrevealed cash flows, future forecasts, or the environment's random-generator state.

The MILP has continuous transfer amounts, binary route-use variables, and positive and negative terminal balances. Its objective combines terminal yields and shortfall penalties with proportional transfer fees and operation costs. Constraints impose balance conservation, route-amount limits, shared source-surplus and target-gap budgets, and the remaining operation budget. We solve it with HiGHS \citep{huangfu2018parallelizing} through \texttt{scipy.optimize.milp} (SciPy 1.15.3), using presolve and a zero relative-gap target. The controller first selects the largest positive planned transfer, breaking ties lexicographically by route, and clips its amount using the executor's limits. If the selected transfer has no executable amount above the tolerance, it chooses STOP without scanning the remaining planned transfers; otherwise, it executes the clipped transfer and replans. Solver optimality concerns this deterministic surrogate rather than the stochastic sequential problem.

\paragraph{Serial timing protocol.}
Controller timing uses an Apple M4 Max, one CPU worker, batch size one, and single-thread settings for PyTorch, BLAS, and HiGHS. The timing subset contains 20 fixed instances sampled without replacement from the 200-instance five-account test bank. Models are loaded before timing, controller blocks are warmed up, and three rounds randomize block and instance order. For each controller--interface--replicate--instance combination, the median episode time across rounds is averaged equally over instances and controller replicates.

Controller time includes action selection and wrapper processing: policy inference for learned controllers, action sampling for Random, and public-state construction, formulation, and solving for MILP. It excludes model loading, reset, actual environment advancement, auditing, and file writing. These measurements describe serial controller computation, not parallel service latency or training cost; episode times also depend on the number of real decisions.

\subsubsection{Route-Mask Construction and Training}

\paragraph{Design and matched quantities.}
This mechanism follow-up was designed after the earlier five-account controller and ten-account mask results were known. All new route subsets, training seeds, test instances, and statistical rules were fixed before observing this experiment's results. We retain the PPO, observation, reward, and executor described above and vary only the ordered stage-zero route list. True denotes the six-route expiration-based MMDP mask. The design includes five uniformly sampled controls, U1--U5, and five direct-funding-coverage controls, C1--C5. The main text compares the funding-preserving controls; this appendix reports both groups and the original primary contrast. True and every alternative expose six routes plus STOP, with identical local selector bins, STOP decoding interval, and amount-ratio decoding. All settings expose the same eight persistent routes at stage one. The masks remain fixed throughout training and testing.

\paragraph{Direct-funding controls.}
Each C mask retains the three routes from the master account to the operating accounts and selects three additional routes from the other eleven. Excluding True leaves $\binom{11}{3}-1=164$ candidates. Five distinct C masks are drawn using lexicographically ordered candidate lists and NumPy 2.0.1 PCG64 with seeds 982201--982205, rejecting duplicates only. No balance, forecast, cost score, trajectory, or return is used to choose masks, and topology diagnostics do not trigger redrawing.

The coverage-preserving controls retain direct funding rather than merely graph connectivity because the original gap-limited executor does not guarantee multi-hop forwarding. They prevent missing direct access to an operating account from being the sole explanation for the comparison, but do not preserve investment opportunities, reverse flow, balances, or attainable policy value.

\paragraph{Uniform controls.}
After fixing C1--C5, the U candidate pool contains all six-route subsets except True and these five C masks, leaving $\binom{14}{6}-6=2997$ candidates. U1--U5 are sampled uniformly from this lexicographically ordered pool using PCG64 seeds 982101--982105, rejecting only previously selected U masks. These controls match route count but need not preserve direct funding coverage; topology diagnostics do not trigger redrawing.

\paragraph{Complete route sets.}
Let $M$, $I$, and $O_1,O_2,O_3$ denote the master, investment, and operating accounts. The fixed route catalogue, indexed from zero, is
\[
\begin{aligned}
(e_0,e_1,e_2,e_3)&=(M{\to}I,I{\to}M,M{\to}O_1,O_1{\to}M),\\
(e_4,e_5,e_6,e_7)&=(M{\to}O_2,O_2{\to}M,M{\to}O_3,O_3{\to}M),\\
(e_8,e_9,e_{10})&=(O_1{\to}O_2,O_2{\to}O_1,O_1{\to}O_3),\\
(e_{11},e_{12},e_{13})&=(O_3{\to}O_1,O_2{\to}O_3,O_3{\to}O_2).
\end{aligned}
\]
The common stage-one list is $(0,1,8,9,10,11,12,13)$. Within each stage-zero subset, the original catalogue order determines the local action mapping. Table~\ref{tab:cash5-mask-inventory-results} reports all ten fixed control subsets and their final-policy results.

\paragraph{Mask-policy training.}
The Flat and True endpoint policies from the controller comparison are reused. For each alternative mask, we train ten policies with seeds 42--51, giving 100 mask policies and 120 final policies in total. All endpoints are at 500224 transitions; no validation selection, early stopping, warm-starting, or Search is used. Training streams follow the original seed rule and are paired by episode ordinal, not by the transition index of differently sized episodes.

\subsection{Five-Account Evaluation and Complete Results}

\label{app:cash5-semantic-mask}

\subsubsection{Common Evaluation Protocol and Controller Results}

\paragraph{Independent-test evaluation.}
The independent-test bank contains 200 shared instances with environment seeds 980000--980199. Each newly constructed wrapper is explicitly reset before scoring, retaining reset ordinal two. Final PPO policies are evaluated deterministically, without validation-based checkpoint selection. All controllers receive matched exogenous realizations and report the native undiscounted sum of all step rewards in Eq.~\eqref{eq:cash-episode-return}. PPO averages instances within policy before forming ten paired training-seed differences. Random averages ten action-sampling replicates within instance before forming paired instance differences; MILP is deterministic and evaluated once per instance. Mask averages the five C controls before contrasts are formed. The timing subset contains 20 instances rather than the full 200-instance return bank.

\paragraph{Complete controller comparison.}
Table~\ref{tab:cash5-complete-results} collects the returns, paired interface and mask contrasts, and serial controller times underlying Table~\ref{tab:cash5-independent-test}. All prescribed training seeds and action-sampling replicates remain in their corresponding averages. MMDP improves PPO and Random mean returns, while MILP obtains the same transfers and returns under all seven route interfaces.

\begin{table}[htbp]
\centering
\small
\setlength{\tabcolsep}{3pt}
\renewcommand{\arraystretch}{1.1}
\caption{Complete five-account controller comparisons on 200 shared test instances. Mask is the equal average of C1--C5. Parenthesized controller times use 20 shared instances on a single-thread CPU. Paired 95\% CIs resample PPO training seeds or Random/MILP test instances; timing and return statistics have different sampling units.}
\label{tab:cash5-complete-results}
\begin{tabular*}{\linewidth}{@{\extracolsep{\fill}}lccccc@{}}
\toprule
& \multicolumn{3}{c}{Interface comparison} & \multicolumn{2}{c}{Mask comparison} \\
\cmidrule(lr){2-4}\cmidrule(l){5-6}
Controller & \shortstack{Flat\\{\scriptsize (ms/episode)}} & \shortstack{MMDP\\{\scriptsize (ms/episode)}} & \shortstack{MMDP $-$ Flat\\{\scriptsize [95\% CI]}} & \shortstack{Mask\\{\scriptsize (ms/episode)}} & \shortstack{MMDP $-$ Mask\\{\scriptsize [95\% CI]}} \\
\midrule
PPO & \shortstack{$-552.003$\\{\scriptsize $(0.581)$}} & \shortstack{$-434.278$\\{\scriptsize $(0.543)$}} & \shortstack{$+117.725$\\{\scriptsize $[85.237,151.115]$}} & \shortstack{$-486.853$\\{\scriptsize $(0.610)$}} & \shortstack{$+52.575$\\{\scriptsize $[41.175,64.955]$}} \\
Random & \shortstack{$-789.646$\\{\scriptsize $(0.190)$}} & \shortstack{$-720.130$\\{\scriptsize $(0.160)$}} & \shortstack{$+69.515$\\{\scriptsize $[64.246,74.890]$}} & \shortstack{$-721.888$\\{\scriptsize $(0.160)$}} & \shortstack{$+1.758$\\{\scriptsize $[-1.107,4.687]$}} \\
MILP & \shortstack{$-330.286$\\{\scriptsize $(2.255)$}} & \shortstack{$-330.286$\\{\scriptsize $(1.988)$}} & \shortstack{$0.000$\\{\scriptsize $[0.000,0.000]$}} & \shortstack{$-330.286$\\{\scriptsize $(2.082)$}} & \shortstack{$0.000$\\{\scriptsize $[0.000,0.000]$}} \\
\bottomrule
\end{tabular*}
\end{table}

\subsubsection{Validation Learning Dynamics}
\label{app:cash-v1-learning-evaluation}

\paragraph{Validation curves and Random references.}
The learning-curve matrix evaluates the twenty Flat/MMDP runs at one hundred scheduled checkpoints, from 5000 through 500000 transitions in 5000-transition increments. Fixed-bank curve scoring is distinct from the continuing ten-episode evaluation callback used during training. Training seed $i$ uses the fixed environment bank $B_i=\{i+1000,\ldots,i+1019\}$, with a newly constructed wrapper explicitly reset before each deterministic evaluation (reset ordinal two). The banks are matched within each Flat/MMDP pair and overlap across training seeds. The released offline scoring entry point evaluates saved scheduled checkpoints without selecting a policy; the final independent-test endpoint remains at 500224 transitions. The learning plots describe across-seed validation dynamics, not independent-test performance.

The Random validation references in Figure~\ref{fig:cash_v1_a} use the same documented banks, environment configuration, wrappers, and reset convention as PPO. The ten banks contain 29 unique environment seeds, 1042--1070. Random uses ten policy-RNG replicates per interface and unique instance, giving $2\times29\times10=580$ evaluations. We first average replicates within each instance, then average the twenty instances within each bank, and finally average the ten bank means equally. This preserves the PPO curves' instance weights rather than uniformly averaging the 29-instance union. These overlapping bank means are not independent training repetitions. Historical aggregate PPO curves do not permit retrospective episode-level fingerprint verification; the Random references match their documented evaluation protocol. The descriptive validation means are $-761.180$ for Random-Flat and $-700.457$ for Random-MMDP.

\subsubsection{Individual Route-Mask Results and Interpretation}

\begin{table}[htbp]
\centering
\small
\setlength{\tabcolsep}{5pt}
\caption{Complete route-mask inventory and test returns. U controls match route count; C controls also preserve direct funding coverage, counted out of three master-to-operating routes. Means average 200 instances within training seed and then ten seeds; SD is across the ten seed means. This is the common bank used in Table~\ref{tab:cash5-independent-test}.}
\label{tab:cash5-mask-inventory-results}
\begin{tabular}{@{}llcrr@{}}
\toprule
Setting & Stage-zero route indices & Funding coverage & Mean & Seed SD \\
\midrule
Flat & $0,\ldots,13$ & $3/3$ & $-552.003$ & $59.071$ \\
True & $2,3,4,5,6,7$ & $3/3$ & $-434.278$ & $10.145$ \\
\midrule
U1 & $1,2,5,6,8,12$ & $2/3$ & $-611.724$ & $54.034$ \\
U2 & $6,7,9,10,12,13$ & $1/3$ & $-631.880$ & $54.546$ \\
U3 & $1,3,4,7,9,10$ & $1/3$ & $-617.332$ & $0.000$ \\
U4 & $1,3,4,10,11,12$ & $1/3$ & $-617.332$ & $0.000$ \\
U5 & $0,3,5,9,11,13$ & $0/3$ & $-623.047$ & $9.202$ \\
\midrule
C1 & $0,2,4,5,6,13$ & $3/3$ & $-510.815$ & $60.256$ \\
C2 & $0,2,3,4,6,7$ & $3/3$ & $-521.022$ & $38.892$ \\
C3 & $2,3,4,6,11,13$ & $3/3$ & $-450.455$ & $9.952$ \\
C4 & $2,4,6,8,12,13$ & $3/3$ & $-461.660$ & $17.019$ \\
C5 & $0,2,4,6,9,11$ & $3/3$ & $-490.316$ & $83.555$ \\
\bottomrule
\end{tabular}
\end{table}

\paragraph{Test coverage.}
The complete results include all 120 PPO endpoint policies and 24000 test episodes from Flat, True, U1--U5, and C1--C5. The main-text Flat, True, and C1--C5 comparison is a 70-policy, 14000-episode subset of this study. All endpoints were fixed before the common test results were accessed.

\paragraph{Estimands and inference.}
Let $R_{i,m}$ be the mean test return for training seed $i$ and interface or mask $m$. The PPO interface contrast is $R_{i,\mathrm{True}}-R_{i,\mathrm{Flat}}$, and the Mask contrast is $R_{i,\mathrm{True}}-\frac{1}{5}\sum_{k=1}^{5}R_{i,C_k}$. Mask averaging precedes inference across the ten paired training seeds. Both marginal 95\% intervals resample complete paired training-seed rows; the Mask interval uses 100000 percentile-bootstrap resamples with PCG64 analysis seed 982301. Neither episodes nor individual mask policies are treated as independent training repetitions. The True-minus-C comparison remains a pre-specified secondary contrast; focusing the main text on C does not change that designation.

The original sole primary contrast is $R_{i,\mathrm{True}}-\frac{1}{10}\sum_{k=1}^{5}(R_{i,U_k}+R_{i,C_k})$. Its mean is $119.280$ (paired-bootstrap 95\% CI $[112.272,126.801]$), with positive differences in all ten training-seed pairs. The interval uses 100000 complete-row percentile-bootstrap resamples with PCG64 seed 982301. A two-sided sign-flip test enumerating all $2^{10}$ assignments, including ties, gives $p=0.001953125$, conditional on sign-exchangeability of the paired differences. The primary is tested separately; the original Holm family comprises the three secondary contrasts True-minus-Flat, True-minus-mean(U), and True-minus-mean(C).

Inference is conditional on the fixed masks and test bank; it does not establish uniform superiority over all possible route masks.

\paragraph{Interpretation and implementation scope.}
True has a higher mean return than the average of the funding-preserving controls. This supports the importance of retained route identities beyond their count for this PPO implementation. It does not separately identify expiration semantics from account roles or network topology. The continuous local selector maps the same output coordinate to different physical routes across masks, so this is not a fixed-global-logit mask ablation. Equal STOP decoding intervals also do not force equal learned termination probabilities. Flat has fourteen routes rather than six, and its contrast additionally changes selector binning.

The frozen generator assigns the master account zero forecast demand, and the gap-limited executor preserves nonnegative balances before settlement. Consequently, operating-to-master and investment-to-master selections have zero preterminal transfer amounts but remain legal actions incurring the non-STOP penalty. These routes were neither removed nor repaired. In particular, True's reverse routes must not be interpreted as effective cash-recycling paths. The C comparison controls direct funding coverage, not the attainable value of the remaining route set.

\paragraph{Integrity.}
All 120 policy runs and 24000 test episodes are retained. Shared latent instances, deterministic replay, and return reconstruction were verified.

\paragraph{Sensitivity to immediate termination.}
On the common test bank used in Table~\ref{tab:cash5-independent-test}, Flat policies with training seeds 46--48 immediately terminate both stages on all 200 instances; no MMDP policy exhibits this behavior on every instance. All ten pairs remain in the main results. A post-hoc sensitivity analysis excluding these three matched pairs yields a mean MMDP-minus-Flat difference of $92.052$ (paired-bootstrap 95\% CI $[65.345,127.692]$), positive in all seven remaining pairs. The interval uses 100000 percentile-bootstrap resamples of the seven paired training-seed rows with PCG64 seed 2026092201, conditional on the fixed test bank. This behavior-conditioned analysis is not a corrected primary estimate, but shows that the positive gap is not confined to the three always-terminating Flat policies. Among the funding-preserving controls, C1 and C2 at seed 45 and C5 at seed 47 also immediately terminate both stages on every instance; all remain in the reported mask averages.

\subsubsection{Continuous Selector Alignment}
\label{app:cash5-common-representation}

\paragraph{Original and aligned selectors.}
The original policies divide the clipped selector coordinate $[-1,1]$ into equal-width bins for the eligible routes and a final STOP bin. At stage zero, Flat uses 14 routes and MMDP uses six, giving MMDP wider bins and a lower STOP threshold. This change in representation can contribute to the original MMDP gain. To assess its contribution, we freeze the original Flat policies and retrain MMDP using Flat's bin boundaries and route mapping. The aligned policy preserves MMDP's six eligible routes by conditioning the Gaussian selector on their Flat bins and the Flat STOP bin. PPO uses the resulting normalized log density. All three settings share the same late-stage selector, amount decoder, and transfer executor.

\paragraph{Training and evaluation.}
We retrain only the aligned MMDP policies, using the original PPO settings and seeds 42--51 for 500,224 transitions. The original Flat and MMDP results come from their archived final policies. All three sets of final policies are evaluated deterministically on the same previously used 200-instance test bank. Contrasts pair the ten training seeds.

\begin{table}[htbp]
\centering
\small
\caption{Five-account continuous selector configurations and mean test returns. SD is the sample standard deviation of the ten training-seed means.}
\label{tab:cash5-selector-alignment}
\begin{tabularx}{\linewidth}{@{}lYYY@{}}
\toprule
 & Original Flat & Original MMDP & Aligned MMDP \\
\midrule
Early eligible routes & All 14 routes & Six expiring routes & Six expiring routes \\
Early selector bins & 15 equal-width bins & 7 equal-width bins & Flat's 15 bins, with six route bins and STOP eligible \\
Early STOP threshold & $13/15$ & $5/7$ & $13/15$ \\
Early Gaussian selector & Unconditioned & Unconditioned & Conditioned on eligible bins \\
\midrule
Mean test return (SD) & $-552.003\ (59.071)$ & $-434.278\ (10.145)$ & $-498.126\ (25.880)$ \\
\bottomrule
\end{tabularx}
\end{table}

\paragraph{Results and interpretation.}
Table~\ref{tab:cash5-selector-alignment} shows that alignment lowers MMDP's mean return by $63.848$ (paired-bootstrap 95\% CI $[49.640,77.190]$), accounting for part of the original $117.725$ gap over Flat. With the aligned selector, MMDP still outperforms the same Flat policies by $53.877$ ($[19.774,89.389]$). The ablation matches the bin geometry while preserving the interface's eligibility restriction through conditional sampling. In the post-hoc subset of seven seed pairs defined above, the aligned gap is $30.722$ ($[-1.422,67.665]$).

\subsubsection{Random Interface and Mask Comparisons}
\label{app:cash5-common-bank}

\paragraph{Alignment and replication.}
Table~\ref{tab:cash5-independent-test} uses the route-mask bank, environment seeds 980000--980199, for every column. We denote the MMDP mask by True and the five equal-size masks preserving all three master-to-operating funding routes by C1--C5. Mask is the equal average of C1--C5, not the best-performing mask. A separately locked Random follow-up evaluates Flat, True, and C1--C5 with ten action-sampling replicates per instance, yielding $7\times200\times10=14000$ episodes. It uses the Random rule described above together with the unchanged executor, ordered masks, and reset convention, without training or loading PPO policies. The ten replicates represent action-sampling randomness, not training seeds. True versus C matches route count and per-decision STOP probability, but not necessarily episode length or termination timing.

\paragraph{Inference for Random.}
Random returns are first averaged over replicates within each instance and then over the five C masks where applicable. Its two pre-specified contrasts, True minus Flat and True minus mean(C), use 100000 paired environment-bootstrap resamples and 100000-draw two-sided Monte Carlo sign-flip tests, with PCG64 seeds 2026091702 and 2026091703, respectively. The tests include ties and the $(b+1)/(B+1)$ correction, require sign symmetry under the null, and use Holm adjustment over these two contrasts only. The mean gaps are $69.515$ with marginal 95\% CI $[64.246,74.890]$ and $1.758$ with CI $[-1.107,4.687]$; Holm-adjusted $p$-values are $0.000020$ and $0.239348$. Positive instance counts are 196/200 and 101/200. PPO intervals instead resample paired training-seed rows; they are not replaced by environment-level inference.

\paragraph{Provenance and scope.}
The Random follow-up was designed after prior results were known and reuses the PPO test bank; it is not a new-bank replication. All planned episodes are retained. The inconclusive Random mask contrast does not establish equivalence or a statistically resolved difference from PPO's interface response. Neither comparison isolates expiration semantics from route roles, topology, and execution.

\subsubsection{Forecast-Based Optimization Reference and Runtime}
\label{app:cash5-or-reference}

\paragraph{Five-account comparison.}
Flat admits every currently available route; MMDP admits the stage-local routes; Mask averages the five fixed funding-preserving masks C1--C5. The forecast processing, MILP, action sequencing, executor, and reward are unchanged. The MILP controller is evaluated once on each of the 200 common test instances, with environment seeds 980000--980199 and matched exogenous realizations. Difference intervals use 100000 paired percentile-bootstrap resamples of instances with PCG64 seed 2026092302, conditional on the fixed masks. These intervals quantify instance variation; PPO intervals instead resample training seeds.

The MILP controller obtains mean return $-330.286$ under Flat, MMDP, and every funding-preserving mask. Its semantic actions, executed amounts, and balance paths coincide across all seven route lists: at stage zero it executes only the three direct funding routes common to every list, and the stage-one lists are identical. Both MMDP-minus-Flat and MMDP-minus-Mask return differences are therefore zero, with intervals rounding to $[0.000,0.000]$. The executor clips stage-zero funding amounts because planning uses shock-adjusted demand while execution uses the visible forecast gap; this occurs identically across interfaces. This result describes the frozen controller on this test bank and does not imply that the interfaces are generally equivalent.

\paragraph{Five-account controller timing.}
PPO includes all ten training seeds, Random includes ten action-sampling replicates per instance, and MILP is evaluated once per instance per round, giving 8820 timed episodes under the shared serial protocol. Mask averages the five C masks equally. Table~\ref{tab:cash5-independent-test} reports mean time per episode alongside returns; timing uses the 20-instance subset, whereas return means use all 200 instances. Timed executions reproduce their corresponding archived returns, episode lengths, and instance fingerprints.

\subsection{Ten-Account Methods and Training}
\label{app:cash10-methods}

The ten-account controllers use the same environment, observation fields, and reward within each interface comparison. PPO is evaluated with and without AA; Random and MILP provide non-learning references. Search is an execution-time procedure applied to frozen PPO policies rather than part of training.

\paragraph{Shared training protocol.}
We compare Flat, MMDP, Flat-AA, and MMDP-AA using ten matched training seeds, 42--51, with forty final policies in total. All receive the same 53-dimensional observation, including the stage indicator, and use separate policy and value networks with two 256-unit hidden layers and Tanh activations. PPO uses a constant learning rate of $3\times10^{-4}$, rollout length 512, minibatch size 128, ten optimization epochs, discount factor 0.99, GAE parameter 0.95, clipping parameter 0.2, entropy coefficient 0.001, value-loss coefficient 0.5, and maximum gradient norm 0.5. Advantages are normalized; observations and rewards are not normalized. Policies use orthogonal initialization and Adam with $\epsilon=10^{-5}$ and coefficients $(0.9,0.999)$, with one training environment per policy and CPU float32 execution.

The nominal training budget is two million transitions. Completing the final rollout and PPO update gives an actual endpoint of 2,000,384 transitions for every policy. Evaluation uses these final updated policies, not validation-selected checkpoints or the callback checkpoint at exactly 2,000,000 transitions. No Search or distillation is used in this training cohort.

\paragraph{Action representations and execution.}
Basic policies output a continuous route selector and amount ratio; AA uses categorical route slots and the forecast-based quantity in Eq.~\eqref{eq:section4-appendix-amount} without a learned ratio. A permitted route with zero executable amount remains a transfer action: it consumes a step, incurs the base environment's penalties, and follows the usual boundary rules rather than automatically selecting STOP. This \texttt{zero\_passthrough} rule applies during both Raw execution and Search rollouts.

\paragraph{STOP aliases and observed behavior.}
Flat-AA has 75 categorical output slots: 74 route slots and one STOP slot at the early stage, and 58 route slots and 17 STOP aliases at the late stage. MMDP-AA has 59 output slots: 16 route slots and 43 STOP aliases at the early stage, and 58 route slots and one STOP slot at the late stage. These counts describe the encoding, not the learned action probabilities. The comparison therefore includes both route availability and categorical representation; matched hidden-layer widths and PPO settings do not equalize output dimensions, parameter counts, or effective entropy regularization. Cross-AA comparisons measure complete implementations rather than an isolated executor effect.

A post-hoc analysis of the saved Raw evaluation records finds no ten-account policy that immediately terminates both stages on every test instance. However, Flat-AA and MMDP-AA execute averages of $19.173$ and $10.023$ non-STOP actions per episode, including $14.718$ and $4.651$ zero-amount transfers, respectively. These quantities average within policy over the shared test instances and then equally over the ten training seeds. Zero-amount transfers remain non-STOP actions under the evaluated executor.

\paragraph{Optimization and timing methods.}
The rolling-horizon MILP uses current balances, visible forecasts, public model coefficients, and interface-admitted routes. Stage-zero demand includes the public expected localized shock; stage one uses the revealed forecast. Continuous amounts and binary route-use decisions minimize transfer costs and terminal shortfall costs net of yields, subject to balance conservation, amount limits, shared source and target budgets, and the remaining operation budget. HiGHS through \texttt{scipy.optimize.milp} (SciPy 1.15.3) solves the deterministic surrogate with presolve and a zero relative-gap target. The controller first selects the largest positive planned transfer, breaking ties lexicographically by route, and clips its amount using the executor's limits. If the selected transfer has no executable amount above the tolerance, it chooses STOP without scanning the remaining planned transfers; otherwise, it executes the clipped transfer and replans. It does not read unrevealed outcomes or environment randomness.

Serial timing uses an Apple M4 Max, one CPU worker, batch size one, and single-thread PyTorch, BLAS, and HiGHS. Models are preloaded, controller blocks warmed up, and three rounds randomize block and instance order on 20 fixed test instances. Repetition medians are averaged equally over instances and controller replicates. Timing includes policy inference, wrapper processing, Random sampling, or MILP construction and solving, as applicable. Search additionally includes public-state construction and hypothetical environment advancement. Loading, reset, actual environment advancement, auditing, and file writing are excluded; this measures serial computation rather than parallel service latency.

\subsection{Ten-Account Evaluation, Search, and Planning Budgets}
\label{app:cash10-evaluation}

This subsection documents the ten-account learning comparison in Figure~\ref{fig:cash_v2_b}, the Raw/Search evaluation in Table~\ref{tab:cash-complete-comparison}, and the Random, MILP, and runtime references. The learned-policy comparisons reuse forty final policies on 200 common test instances; Random and MILP use the same instance bank but not those learned policies.

\subsubsection{Learning with Action Abstraction}
\label{app:cash10-learning-dynamics}

\paragraph{Return-component accounting.}
The observed Raw MMDP-AA advantage includes both fewer penalized actions and a higher terminal yield-minus-shortfall outcome. This is trajectory accounting, not a causal decomposition: actions, balances, and terminal outcomes change jointly.

\paragraph{Validation learning curves.}
For training seed $i$, each validation point averages twenty complete episodes with environment seeds $i+1000+j$, $j=0,\ldots,19$. Banks remain fixed during training and are shared across families, although banks for different training seeds may overlap. Wrapper construction performs the first reset and an explicit reset starts the scored second episode; all validation runs retain this reset convention. Validation returns use per-step float32 accumulation. The observed grid is every 5,000 transitions from 5,000 through 2,000,000, giving 400 points per policy and 16,000 points across the forty policies. Curves show means and sample standard deviations across ten training-seed means, without smoothing, best-so-far transformations, or an invented time-zero observation. The two Random references shown with the learning curves use the corresponding fixed validation banks and are not independent-test baselines.

The two-million-transition budget was chosen after inspecting earlier validation curves. The cohort therefore does not establish convergence or a blind comparison of training budgets.

\paragraph{Independent test and statistical units.}
All forty final policies are evaluated on 200 shared test instances, with environment seeds 994000--994199 and reset ordinal two, giving 8,000 episodes. This bank was held out from training and validation. Evaluation uses deterministic decoded policy actions and the native sum of all step rewards, not terminal reward alone. Returns are first averaged over the 200 instances within each policy, then contrasted within each of the ten matched training seeds. Uncertainty concerns training randomness conditional on this fixed test bank; the 2,000 episodes per family are not treated as 2,000 independent training repetitions.

The four primary raw-policy contrasts are listed in Table~\ref{tab:cash10-config-a-contrasts}. Marginal percentile-bootstrap 95\% intervals use 100,000 whole-training-seed resamples, shared across contrasts, with PCG64 seed 2026091409 and linear percentile interpolation. Exact two-sided paired sign-flip tests enumerate all $2^{10}=1024$ sign assignments, including ties with absolute numerical tolerance $10^{-12}$. Their inferential interpretation requires paired sign-exchangeability or symmetry. Holm adjustment covers these four raw-policy tests; the displayed confidence intervals are individually unadjusted.

\begin{table}[htbp]
\centering
\small
\caption{Final raw-policy results on 200 shared independent-test instances. SD is across ten training-seed mean returns. Output dimensions and parameter counts disclose implementation differences despite shared hidden-layer widths and PPO settings.}
\label{tab:cash10-config-a-endpoints}
\begin{tabular*}{\linewidth}{@{\extracolsep{\fill}}lrrrr@{}}
\toprule
Family & Output dimension & Parameters & Mean return & SD \\
\midrule
Flat & 2 & 160005 & $-712.639$ & $91.812$ \\
MMDP & 2 & 160005 & $-567.546$ & $148.361$ \\
Flat-AA & 75 & 178764 & $-546.925$ & $27.923$ \\
MMDP-AA & 59 & 174652 & $-388.297$ & $11.805$ \\
\bottomrule
\end{tabular*}
\end{table}

\begin{table}[htbp]
\centering
\small
\caption{Four primary raw-policy contrasts. Confidence intervals are marginal paired-bootstrap 95\% intervals; Holm adjustment covers the four tests in this table. Positive pairs count training seeds with a positive difference.}
\label{tab:cash10-config-a-contrasts}
\begin{tabular*}{\linewidth}{@{\extracolsep{\fill}}lrlrr@{}}
\toprule
Contrast & Mean gap & 95\% CI & Holm $p$ & Positive pairs \\
\midrule
MMDP $-$ Flat & $145.092$ & $[45.270,228.854]$ & $.017578$ & 9/10 \\
MMDP-AA $-$ Flat-AA & $158.628$ & $[143.495,173.972]$ & $.007813$ & 10/10 \\
Flat-AA $-$ Flat & $165.714$ & $[108.700,217.625]$ & $.007813$ & 10/10 \\
MMDP-AA $-$ MMDP & $179.249$ & $[121.004,276.249]$ & $.007813$ & 10/10 \\
\bottomrule
\end{tabular*}
\end{table}

AA improves both policy families in this cohort, and MMDP-AA also exceeds Flat-AA. All training seeds are retained, including the negative non-AA interface difference for seed 48.

\subsubsection{Non-Learning Random Interface Reference}
\label{app:cash10-random-reference}

The Random controller uniformly samples one currently exposed semantic route or STOP, then independently samples a transfer ratio uniformly on $[0,1]$. This post-result fixed-controller follow-up uses the same 200 independent-test instances (environment seeds 994000--994199, reset ordinal two) as the ten-account PPO and MILP comparisons. Ten action-sampling replicates per instance and interface yield 4000 complete episodes; these replicates are not training seeds. Within each instance--replicate pair, Flat and MMDP start from the same action-RNG seed, although their streams may subsequently diverge when episode lengths differ. No policy is trained or selected for this comparison.

Returns sum all native step rewards. We first average the ten replicates within each instance and interface, then form 200 paired MMDP-minus-Flat instance differences. A 100000-resample percentile bootstrap over instances (PCG64 seed 2026092402) yields the primary 95\% CI, conditional on the sampled action-RNG replicates. Random-Flat and Random-MMDP have mean returns of $-983.908$ and $-855.267$, respectively, giving a gap of $+128.641$ (95\% CI $[121.287,136.123]$).

The stage-zero STOP probability changes from $1/75$ under Flat to $1/17$ under MMDP, along with the exposed routes. The result therefore measures this complete interface response for a non-learning controller; it does not isolate expiration semantics. Its primary interval resamples test instances, whereas the PPO intervals resample training seeds, so their widths do not quantify the same source of variation.

\subsubsection{Policy-Guided Search with Frozen Policies}
\label{app:cash-search-evaluation}
\label{app:cash-dfm-evaluation-v7}

\paragraph{Policies, instances, and planning caps.}
Search uses the same forty final policies and 200 instances as Raw, retaining its 8,000 episode returns. The three per-decision caps are $1\times$, $2\times$, and $4\times$ the base budget; $4\times$ is primary. Policy parameters remain frozen, with no fine-tuning, checkpoint selection, or distillation. The Search protocol was locked after earlier Raw and Search results were known but before the present Search outcomes were accessed; the bank is held out from training, not previously unseen by the investigators.

\paragraph{Public-information planning worlds.}
The planner constructs deterministic forecast-mean scenarios using only the current public observation and known environment parameters. At stage zero, it uses one no-shock world and eight target-shock worlds, with the public mean positive shock magnitude and occurrence/target probabilities. The visible forecast supplies a point estimate of base demand. At stage one, it uses one world whose terminal outflow equals the currently visible updated forecast. The controller does not copy the real episode's unrevealed shock, future forecast, terminal outflow, or random-generator state. Hidden trajectory fingerprints are used only by the auditor. This is a forecast-mean approximation, not exact posterior integration or a scenario-sampled stochastic optimizer.

\paragraph{Unified candidate construction.}
Within the non-AA pair and within the AA pair, candidate generation, pruning, ordering, scoring, and cap admission use the same program. Legal route sets and the frozen policy outputs still differ. Table~\ref{tab:cash-candidates-v7} gives the common candidate upper bounds. These are separate from the per-decision planning-transition cap and do not require the planner to fill every candidate slot.

\begin{table}[htbp]
\centering
\small
\caption{Candidate rules shared within each interface pair. AA uses its existing amount executor; the continuous amount grid is not applied to AA.}
\label{tab:cash-candidates-v7}
\begin{tabular*}{\linewidth}{@{\extracolsep{\fill}}llr@{}}
\toprule
Interface pair & Amount candidates & Candidate upper bound \\
\midrule
Flat / MMDP & $0.25,0.5,0.75,1.0$ and policy ratio & 32 \\
Flat-AA / MMDP-AA & Existing AA executor & 16 \\
\bottomrule
\end{tabular*}
\end{table}

Non-AA candidates are formed from legal routes and the amount-ratio grid, rounding ratios to six decimal places, clipping, and removing duplicates. If the full pool exceeds the limit, the original prefilter retains the policy route and its immediate local-index neighbors; a STOP proposal instead uses the first three routes. Consequently, the nominal cap of 32 commonly yields at most 16 proposals, including STOP. The neighborhood is not expanded to fill unused slots. AA candidates use probability-ranked legal routes, reserve a legal policy proposal, and include one semantic STOP. Original slot decoding and tie behavior are preserved; probabilities of STOP aliases are not merged before policy argmax. The runtime rejects invalid or missing categorical scores rather than silently changing to a fallback proposal rule.

\paragraph{Scoring and budget admission.}
For each candidate, the planner executes its first action and follows the deterministic frozen policy to termination in every hypothetical world. Scores combine the full undiscounted sums of step rewards with the world weights, using the original wrapper and executor. The retained policy proposal is scheduled first and STOP second, unless STOP is itself the policy proposal. Before admitting a candidate, the planner checks whether the work already consumed plus the number of worlds times the maximum remaining episode length fits within the cap. Each admitted candidate must complete its entire world bundle; no partial score or terminal-value bootstrap is used. Actual simulated transitions are charged after execution. Exact score ties retain the original STOP-first semantic rule. Only the chosen first action is executed in the real environment, followed by replanning at the next observation.

\paragraph{Primary estimands and inference.}
The two primary contrasts, both at the $4\times$ budget, are H1: MMDP-AA Search minus Flat-AA Search, and H2: MMDP-AA Search minus MMDP-AA Raw. Episodes are averaged within policy before forming ten matched training-seed differences. Marginal paired-bootstrap 95\% intervals use 100,000 whole-seed resamples, shared across all families, caps, and contrasts, with PCG64 seed 2026091501 and linear percentile interpolation. Exact two-sided sign-flip tests enumerate all 1024 assignments with inclusive ties and tolerance $10^{-12}$; the same sign-exchangeability qualification as above applies. Holm adjustment covers only H1 and H2, not the separate four-test raw-policy family. No multiplicity adjustment is applied to the reported intervals, and no secondary $p$-values are added.

H1 is $80.193$ with 95\% CI $[71.324,89.623]$; H2 is $94.611$ with CI $[88.645,100.735]$. Both are positive in all ten training-seed pairs, with exact $p=0.001953125$ and Holm-adjusted $p=0.00390625$. The uncertainty is conditional on the shared 200-instance bank.

\paragraph{Search results.}
Table~\ref{tab:cash10-search-primary-results} reports all four families' primary-budget returns and gains. Search improves all four families, with positive within-seed gains in every training-seed pair. MMDP-AA retains the highest mean Search return among these four learned policy families. The non-AA MMDP-minus-Flat Search gap is $150.108$ (95\% CI $[82.028,214.718]$). The AA interface gap is $80.193$, compared with $158.628$ under Raw; the observed mean gap therefore narrows with Search.

\begin{table}[htbp]
\centering
\small
\setlength{\tabcolsep}{3pt}
\renewcommand{\arraystretch}{1.15}
\caption{Complete ten-account Raw and Search results at the primary $4\times$ budget on 200 shared instances, using the same frozen policies. Raw and Search brackets are marginal bootstrap 95\% CIs over ten training seeds; Search-gain brackets use paired within-seed differences. These are the primary-budget results underlying Table~\ref{tab:cash-complete-comparison}.}
\label{tab:cash10-search-primary-results}
\begin{tabular*}{\linewidth}{@{\extracolsep{\fill}}lccr@{}}
\toprule
Policy & Raw & Search & \multicolumn{1}{c@{}}{Search gain} \\
\midrule
Flat & $-712.64\;{\scriptstyle[-765.65,-654.86]}$ & $-604.80\;{\scriptstyle[-656.28,-547.31]}$ & $+107.84\;{\scriptstyle[94.73,120.96]}$ \\
MMDP & $-567.55\;{\scriptstyle[-665.37,-509.30]}$ & $-454.69\;{\scriptstyle[-517.07,-414.68]}$ & $+112.85\;{\scriptstyle[87.24,151.90]}$ \\
Flat-AA & $-546.93\;{\scriptstyle[-564.09,-531.37]}$ & $-373.88\;{\scriptstyle[-381.99,-365.32]}$ & $+173.05\;{\scriptstyle[156.05,191.69]}$ \\
MMDP-AA & $-388.30\;{\scriptstyle[-395.33,-381.45]}$ & $-293.69\;{\scriptstyle[-298.21,-289.54]}$ & $+94.61\;{\scriptstyle[88.65,100.74]}$ \\
\bottomrule
\end{tabular*}
\end{table}

\subsubsection{Optimization Reference and Serial Runtime}
\label{app:cash10-or-reference}

\paragraph{Optimization comparison.}
We apply the shared forecast-based rolling-horizon MILP to the ten-account environment. Within this environment, Flat and MMDP use the same forecast processing, MILP formulation, sequencing rule, executor, and reward. Both are evaluated on the 200 common test instances with environment seeds 994000--994199 and matched exogenous realizations. A 100000-resample paired percentile bootstrap over instances (PCG64 seed 2026092301) gives the interval for the interface difference.

Mean return is $-216.312$ for MILP-Flat and $-206.290$ for MILP-MMDP, a paired improvement of $10.022$ (95\% CI $[6.579,13.853]$). The paired differences are positive, zero, and negative on 45, 143, and 12 instances, respectively. Both MILP means exceed those of the evaluated learned controllers on this bank, including their Search-assisted versions. This is a comparison with one deterministic rolling-horizon controller, not a bound on stochastic optimality.

\paragraph{Serial controller runtime.}
\label{app:cash-serial-runtime}
The frozen policies, MILP controller, and Search implementations are measured on 20 common test instances under the shared serial CPU protocol. Search uses the primary cap of 2304 simulated transitions per decision; its timing includes public-state construction, hypothetical rollouts, and action selection. Episode times are averaged equally over instances and ten training seeds for learned controllers, or over instances alone for MILP. Per-decision values pool visited decisions after taking repetition medians.

\begin{table}[htbp]
\centering
\small
\setlength{\tabcolsep}{3pt}
\renewcommand{\arraystretch}{1.1}
\caption{Ten-account serial CPU controller time in milliseconds on 20 shared test instances. Raw denotes direct policy execution; Search uses the primary planning cap. Values follow the aggregation described above. Lower is better.}
\label{tab:cash-serial-runtime}
\begin{tabular*}{\linewidth}{@{\extracolsep{\fill}}lrrrr@{}}
\toprule
& \multicolumn{2}{c}{Per decision} & \multicolumn{2}{c}{Per episode} \\
\cmidrule(lr){2-3}\cmidrule(l){4-5}
Controller & Flat & MMDP & Flat & MMDP \\
\midrule
PPO Raw & $0.082$ & $0.081$ & $0.888$ & $1.535$ \\
PPO-AA Raw & $0.083$ & $0.086$ & $1.720$ & $1.058$ \\
PPO + Search & $96.620$ & $109.764$ & $740.795$ & $1131.344$ \\
PPO-AA + Search & $136.311$ & $88.059$ & $1965.039$ & $1019.153$ \\
MILP & $2.366$ & $0.834$ & $25.094$ & $8.941$ \\
\bottomrule
\end{tabular*}
\end{table}

The MILP controller has higher mean test return than the learned controllers and higher measured controller time than direct policy execution; serial Search is more computationally expensive than MILP. Under MILP, the MMDP interface has both a higher mean return and a lower mean episode time than Flat. These measurements provide a common single-thread CPU reference, not an optimized deployment-latency comparison. Search can exploit parallel continuation-rollout evaluation and batched policy inference, with achievable acceleration depending on the simulator and hardware. Multicore and GPU configurations are not benchmarked here, so the observed serial runtime ordering should not be extrapolated to those configurations or to the separate production controller. End-to-end service latency and training cost are outside this timing protocol.

\subsubsection{Planning-Budget Sensitivity and Computational Accounting}
\label{app:cash-planning-budget}

\paragraph{Budget definition.}
Let $B_0=576$ be the base cap on simulated transitions per real decision. The evaluated Search budgets are $B=kB_0$ for $k\in\{1,2,4\}$; Raw invokes no Search. A simulated transition is one environment step taken while evaluating a candidate continuation in a hypothetical world, not a training transition, a complete rollout, or a real transfer. At stage zero, the planner uses nine worlds and a conservative maximum remaining length of 32 transitions per world. Two complete worst-case candidate bundles therefore require at most $2\times9\times32=576$ transitions. Together with policy-first and STOP-second scheduling, the base cap permits both retained proposals to be scored completely. The larger caps double and quadruple this allowance; these are protocol settings, not theoretically optimal budgets.

\paragraph{Complete budget grid.}
Raw execution and all three Search caps retain the same policy cohort and test instances. The assembled comparison contains sixteen family-by-budget settings, 160 setting-by-training-seed endpoints, and 32,000 episode rows, including the 8,000 reused raw rows. Table~\ref{tab:cash-all-caps-v7} reports all mean returns. The figure uses equally spaced condition labels Raw, $1\times$, $2\times$, and $4\times$. Their positions are not a linear scale of realized computation; connecting lines only guide the eye between evaluated conditions. Shading is the sample SD across ten training-seed means, not episode-level SD or a confidence interval.

\begin{table}[htbp]
\centering
\small
\caption{Complete mean returns on the common 200-instance bank. Raw uses no planning; Search budgets are multiples of $B_0=576$ simulated transitions per real decision, with $4\times$ primary. Values average instances within policy and then the ten training-seed means.}
\label{tab:cash-all-caps-v7}
\begin{tabular*}{\linewidth}{@{\extracolsep{\fill}}lrrrr@{}}
\toprule
Family & Raw & $1\times$ & $2\times$ & $4\times$ \\
\midrule
Flat & $-712.639$ & $-615.221$ & $-607.033$ & $-604.802$ \\
MMDP & $-567.546$ & $-466.195$ & $-455.131$ & $-454.694$ \\
Flat-AA & $-546.925$ & $-369.563$ & $-371.722$ & $-373.879$ \\
MMDP-AA & $-388.297$ & $-293.021$ & $-293.701$ & $-293.686$ \\
\bottomrule
\end{tabular*}
\end{table}

Both AA families obtain substantial Search gains at $1\times$, with no consistent additional gains at larger budgets. For both non-AA families, mean returns rise across the tested budgets. Search ranks additional candidates using returns in forecast-based worlds, which can differ from realized test returns.

\paragraph{Realized planning work.}
Table~\ref{tab:cash-search-realized-compute} reports realized simulated planning transitions per episode. The caps apply to each real decision, whereas the table aggregates work over the full episode. Raw uses no Search, so its zero entries do not imply zero controller cost.

\begin{table}[htbp]
\centering
\small
\setlength{\tabcolsep}{4pt}
\caption{Mean realized simulated planning transitions per episode on the shared test bank, rounded to the nearest transition. Parentheses in the column headings give the per-decision cap. Values average test instances within each policy and then the ten training-seed means.}
\label{tab:cash-search-realized-compute}
\begin{tabular*}{\linewidth}{@{\extracolsep{\fill}}lrrrr@{}}
\toprule
Policy & Raw & \shortstack{$1\times$\\{\scriptsize (576)}} & \shortstack{$2\times$\\{\scriptsize (1152)}} & \shortstack{$4\times$\\{\scriptsize (2304)}} \\
\midrule
Flat & $0$ & $1{,}983$ & $4{,}384$ & $6{,}764$ \\
MMDP & $0$ & $3{,}622$ & $8{,}167$ & $11{,}521$ \\
Flat-AA & $0$ & $2{,}743$ & $10{,}937$ & $20{,}937$ \\
MMDP-AA & $0$ & $3{,}066$ & $7{,}460$ & $9{,}829$ \\
\bottomrule
\end{tabular*}
\end{table}

Equal per-decision caps do not imply equal realized episode-level work. At the primary $4\times$ cap, MMDP-AA uses approximately $53\%$ fewer simulated transitions per episode than Flat-AA while attaining a higher mean return (Table~\ref{tab:cash-all-caps-v7}). This difference is not uniform across budgets: at $1\times$, MMDP-AA uses slightly more transitions than Flat-AA. These are descriptive comparisons of the complete policy--Search implementations, not a compute-matched test or an isolated effect of AA.

\paragraph{Design scope and integrity.}
\label{app:cash10-configuration-record}
The current matrix uses common final policies, unified candidate rules within each interface pair, and the same test bank. It supports comparisons among these complete implementations, not a single-module causal attribution. All prescribed policies, training seeds, and caps are retained; audits verified instance matching, full-return reconstruction, unchanged checkpoints, and complete-bundle accounting.

\section{Production-Scale Cash Management}
\label{app:production-scale}

The production case combines a production-derived simulator, historical backtesting, and human-reviewed deployment observations. We separate the environment and simulator protocol from model scope and the definitions of the economic and operational outcomes.

\subsection{Environment and Simulator Evaluation Protocol}
\label{app:prod_cashflow_env}

\paragraph{Setting and evidence sources.}
The production-derived simulator represents a 27-account cash-management network over 24 hourly decision stages per day. It is distinct from the public two-stage synthetic benchmarks. The operational network contains heterogeneous funding, operating, and interest-bearing accounts, with account-specific liquidity requirements and route-dependent transfer constraints. Account identities, institution-specific route tables, and proprietary operational data are not disclosed.

\paragraph{Hourly decision timeline.}
Each trajectory represents one business day. At the beginning of each hourly stage, realized external inflows and outflows are posted to account balances, forecasts are refreshed, and network availability is updated. Multiple transfers may be executed within a stage; an individual transfer is not itself an hourly stage. The controller may also end a stage without transferring funds. In the simulator, an End action advances to the next hour; operationally, this corresponds to waiting for the next decision window. External cash flows therefore affect liquidity throughout the day rather than being settled only at the terminal stage.

\paragraph{Information and action opportunities.}
The observation reflects current balances, realized cash flows, available forecasts, and current operational constraints. Each hourly boundary brings updated information and the permanent expiry of at least one domain-action opportunity. Some routes permit physical execution only later in the day, but the corresponding operations can be reserved or committed earlier. The semantic opportunity catalogue thus distinguishes earlier commitments for later execution from routes immediately executable at the current hour; balance changes determine current executable quantities. Temporary route closures handled through operator rejection are outside the modeled interface. Shared balances and feasibility constraints couple transfers, so the production comparison evaluates a structured restriction rather than establishing the theorem's exact-deferral conditions throughout the network.

\paragraph{Transfers and daily outcomes.}
Transfer decisions specify a route and an amount subject to the simulator's balance and route constraints. External cash flows are accounted for hourly, while the daily return is settled at the end of the day. Evaluation records liquidity shortfalls and overnight deposit income, with overnight deposit interest forming the principal yield component. These daily outcomes should be distinguished from the amount-based recommendation-adoption metric used after deployment.

\paragraph{Policy families and action abstraction.}
The simulator comparison evaluates Flat and MMDP interfaces with PPO-AA and GRPO-AA, both with and without Search-assisted execution. Under AA, the policy selects a transfer route and the executor determines the transfer amount from the current operational state and constraints. The operational executor and planner differ from their synthetic-benchmark counterparts. The comparisons concern the complete policy, interface, and execution implementations; they do not isolate a single optimizer or executor component.

\paragraph{Group-relative policy training.}
GRPO-AA adapts group-relative policy optimization \citep{shao2024deepseekmath} to simulated transfer trajectories. Each group contains trajectories sampled from the same initial condition. Their returns are used to construct group-relative advantages for updating the policy without a separately learned value baseline. The grouping therefore compares alternative trajectories under a common starting condition rather than directly comparing returns from unrelated days. Group-relative training is distinct from the additional planning performed by Search at execution time.

\paragraph{Search-assisted execution.}
Search holds policy parameters fixed and evaluates candidate transfer plans through continuation rollouts using the single cash-flow forecast available at the current decision. Rollout randomness comes from sampling subsequent policy actions, not from sampling multiple future cash-flow scenarios. Returns guide selection of the next action; after execution, the controller replans using newly observed balances, flows, and forecasts. It does not read the evaluation instance's unrevealed future cash flows. Realized flows are available to the simulator or retrospective evaluator for scoring, not to the controller in advance. This information boundary applies to simulator evaluation, historical backtesting, and deployment, without requiring an unchanged checkpoint throughout deployment.

\paragraph{Training and test instances.}
Operational source data are separated before synthetic training and test instances are generated with added stochastic variation. The resulting training and test sets have a 5:1 size ratio. All eight evaluated pipelines use the same 1,000 test instances, each representing a complete daily trajectory. The source-data separation precedes instance generation rather than splitting generated perturbations of the same source records across training and testing.

\paragraph{Evaluation and reporting.}
The production-simulator comparison uses five training seeds per method. Each policy's return averages the 1,000 common test instances, and Figure~\ref{fig:cash_v3_final} reports the mean of the five policy-level averages. Only aggregate means are available for reporting; uncertainty intervals and significance claims are therefore not supplied. Search-assisted execution is evaluated separately from direct policy execution. Table~\ref{tab:production-five-seed-means} records the eight means underlying the figure.

\begin{table}[htbp]
\centering
\small
\caption{Production-derived simulator mean test returns over five training seeds per method on 1,000 common instances, corresponding to Figure~\ref{fig:cash_v3_final}. Raw denotes execution without Search. Values are descriptive means, without confidence intervals or significance claims. Higher return is better.}
\label{tab:production-five-seed-means}
\begin{tabular*}{\linewidth}{@{\extracolsep{\fill}}lrrrr@{}}
\toprule
& \multicolumn{2}{c}{Raw} & \multicolumn{2}{c}{Search} \\
\cmidrule(lr){2-3}\cmidrule(l){4-5}
Pipeline & Flat & MMDP & Flat & MMDP \\
\midrule
PPO-AA & $-3973.3$ & $-3338.0$ & $-1208.0$ & $-852.8$ \\
GRPO-AA & $-4718.5$ & $-2265.3$ & $-1459.1$ & $-624.8$ \\
\bottomrule
\end{tabular*}
\end{table}

\paragraph{Disclosure and interpretation.}
The simulator and underlying operational records are not publicly released because of confidentiality restrictions. The disclosed protocol therefore does not make these production results fully independently reproducible.

\subsection{Model Scope and Evaluation Definitions}
\label{app:production-backtest}

\paragraph{Model scope and evidence sources.}
The simulator means describe the five-seed evaluation models, not the separate approximately 100-million-parameter operational model. The historical backtest evaluates the estimated economic performance of the MMDP-GRPO-AA controller with Search under historical operational conditions. The subsequent deployment comparison measures human operators' adoption of recommendations from the operational model; it is not a direct deployment test of the simulator-evaluation checkpoints. These analyses have different models, outcomes, and replication units and do not constitute a single causal evaluation.

\paragraph{Computational scope and deployment.}
The MILP reference evaluated on the two-stage synthetic benchmarks is not a production-scale optimization baseline. The operational system instead requires repeated replanning as balances, forecasts, and route availability change across 24 hourly stages. Its Search-assisted MMDP-GRPO-AA controller evaluates continuation rollouts in parallel while keeping policy parameters fixed during execution. The production backtest compares this controller with the executed incumbent workflow, including human adjustments, rather than with the synthetic MILP reference.

\subsubsection{Pre-Deployment Historical Backtest}

\paragraph{Controller and baseline.}
The pre-deployment observation window spans six calendar months and includes 94 eligible business days, excluding weekends and applicable holidays, including extended holiday closures. Calendar month identities are withheld. The Search-assisted MMDP-GRPO-AA controller starts from the same daily opening balances as the historical workflow, retains the historical external inflows and withdrawals, and replaces internal transfer decisions. Its resulting balances are propagated hour by hour through daily settlement rather than reset to historical balances. The baseline is the actually executed workflow, combining incumbent optimization-model recommendations with human adjustments, not the optimization model alone. The controller sees external flows only as they occur; Search uses the contemporaneous forecast, not later realized flows. The backtest metric below is distinct from the simulator's numerical return.

\paragraph{Net-income metric and monthly summaries.}
For workflow $k$, define daily net income as $Y_{m,d}^{k}=I_{m,d}^{k}-C_{m,d}^{k}$, where $I$ is overnight deposit interest and $C$ comprises shortfall financing costs and penalties. Other income and cost components are omitted because their magnitudes are small relative to these two components in the evaluated operation. The daily relative improvement, expressed in percent, is
\begin{equation}
g_{m,d}=100\,\frac{Y_{m,d}^{\mathrm{MMDP}}-Y_{m,d}^{\mathrm{historical}}}{Y_{m,d}^{\mathrm{historical}}}.
\label{eq:production-daily-net-income}
\end{equation}
Historical net income is positive on the evaluated days, so no zero- or negative-denominator adjustment is needed. For month $m$, $n_m$ is the evaluated-day count, $\widehat{\mu}_m$ averages the daily percentages equally, and $e_m$ is the standard error of that mean, not the standard deviation of daily outcomes. Monthly means, standard errors, and intervals use the same dates. The reported monthly 95\% intervals use $\widehat{\mu}_m\pm t_{0.975,n_m-1}e_m$; the supplied intervals are retained rather than recomputed from rounded inputs. This metric is neither a percentage reduction in liquidity shortfalls nor realized post-deployment revenue growth.

\begin{table}[htbp]
\centering
\small
\caption{Pre-deployment backtest summaries. Months are ordered chronologically. Improvements and confidence limits are in percent; standard errors are in percentage points. Monthly intervals are retained from the operational summaries; they are not recomputed from the rounded means and SEs. The full-period estimate and approximate interval are computed from the displayed summaries as described below.}
\label{tab:production-backtest-monthly}
\begin{tabular*}{\linewidth}{@{\extracolsep{\fill}}lrrrc@{}}
\toprule
Period & Days & Mean (\%) & SE (pp) & 95\% CI (\%) \\
\midrule
Month 1 & 20 & $1.07$ & $2.01$ & $[-3.12,5.27]$ \\
Month 2 & 13 & $7.97$ & $3.21$ & $[0.97,14.98]$ \\
Month 3 & 14 & $9.69$ & $3.01$ & $[3.17,16.21]$ \\
Month 4 & 8 & $9.99$ & $3.06$ & $[2.74,17.23]$ \\
Month 5 & 18 & $10.27$ & $9.99$ & $[-10.82,31.37]$ \\
Month 6 & 21 & $11.89$ & $2.85$ & $[5.95,17.84]$ \\
\midrule
Full period & 94 & $8.25$ & $2.17$ & $[3.80,12.70]$ \\
\bottomrule
\end{tabular*}
\end{table}

\paragraph{Aggregation.}
The full-period estimate weights monthly means by the number of evaluated days:
\begin{equation}
N=\sum_{m=1}^{6}n_m=94,\qquad
w_m=\frac{n_m}{N},\qquad
\widehat{\mu}=\sum_{m=1}^{6}w_m\widehat{\mu}_m=8.246.
\label{eq:production-backtest-mean}
\end{equation}
This is the mean daily relative net-income improvement across evaluated days. It is neither an equally weighted mean across months nor the relative improvement in total net income aggregated over the observation window.

\paragraph{Approximate uncertainty.}
Under the working assumption that the monthly mean estimates are independent, the standard error of the weighted estimate is
\begin{equation}
e=\left(\sum_{m=1}^{6}w_m^2e_m^2\right)^{1/2}=2.171
\label{eq:production-backtest-se}
\end{equation}
percentage points. Assigning $n_m-1$ degrees of freedom to each monthly variance estimate gives the Welch--Satterthwaite approximation
\begin{equation}
\nu_{\mathrm{eff}}
=
\frac{\left(\sum_{m=1}^{6}w_m^2e_m^2\right)^2}
{\sum_{m=1}^{6}w_m^4e_m^4/(n_m-1)}
\approx27.61.
\label{eq:production-backtest-df}
\end{equation}
The resulting approximate 95\% interval is
\begin{equation}
\widehat{\mu}\pm t_{0.975,\nu_{\mathrm{eff}}}e
=[3.796,12.696]\%.
\label{eq:production-backtest-ci}
\end{equation}
This interval is reconstructed from rounded monthly summaries rather than a day-level bootstrap. It is not a training-seed confidence interval. It relies on the supplied standard errors and the stated independence and degrees-of-freedom assumptions, without an additional correction for temporal dependence or cross-month covariance. Monthly intervals are not adjusted for multiple comparisons; in particular, those for Months 1 and 5 include zero.

\subsubsection{Human-Reviewed Deployment}

\paragraph{Operational comparison.}
The incumbent workflow operated on the same 94 eligible days used for the backtest. Following the system-wide transition, operators reviewed recommendations from the MMDP-GRPO-AA controller with Search, accepting them or replacing them with manually determined transfers. Post-deployment summaries cover 135 eligible days across five complete months and a partial sixth month through the reporting cutoff. Calendar month identities are withheld. The accounting definition and human-review procedure are unchanged between periods. No account-level treatment/control split was used because decisions couple liquidity across the network; the comparison is observational rather than randomized.

\paragraph{Amount-based adoption.}
For day $d$, let $A_d$ be the executed transfer amount attributable to accepted algorithmic recommendations, and let $H_d$ be the manually determined executed amount, including revisions after rejected recommendations and independently added manual transfers. The daily adoption rate is
\begin{equation}
a_d=\frac{A_d}{A_d+H_d},
\qquad A_d+H_d>0.
\label{eq:production-adoption-rate}
\end{equation}
Here $H_d$ uses the final executed manual amount, not the rejected recommendation's amount. Each transfer is counted once, not separately at its sending and receiving accounts. Canceled recommendations and approved but unexecuted transfers contribute no executed amount. Recommendations are accepted or rejected; partial acceptance is not used in this workflow. All observed days have positive total transfer amount. The denominator therefore covers actual algorithm-adopted and manually determined transfers, not the total originally recommended amount; adoption is not an acceptance rate by recommendation count.

\begin{table}[htbp]
\centering
\small
\caption{Mean daily amount-based adoption before and after deployment. Pre-deployment observations cover the same eligible days as the backtest. Post-deployment Month 6 is observed through the reporting cutoff. Calendar identities are withheld. Period means weight monthly means by eligible-day counts.}
\label{tab:production-adoption-monthly}
\begin{tabular*}{\linewidth}{@{\extracolsep{\fill}}llrr@{}}
\toprule
Workflow & Period & Days & Mean adoption (\%) \\
\midrule
\shortstack[l]{Incumbent optimization\\+ human review} & Month 1 & 20 & $19.43$ \\
 & Month 2 & 13 & $12.29$ \\
 & Month 3 & 14 & $15.35$ \\
 & Month 4 & 8 & $20.75$ \\
 & Month 5 & 18 & $16.46$ \\
 & Month 6 & 21 & $48.71$ \\
\cmidrule(lr){2-4}
 & Full pre-deployment period & 94 & $23.92$ \\
\midrule
\shortstack[l]{MMDP-GRPO-AA + Search\\+ human review} & Month 1 & 23 & $73.07$ \\
 & Month 2 & 26 & $74.39$ \\
 & Month 3 & 26 & $72.77$ \\
 & Month 4 & 23 & $76.84$ \\
 & Month 5 & 24 & $70.32$ \\
 & Month 6 (partial) & 13 & $78.36$ \\
\cmidrule(lr){2-4}
 & Observed post-deployment period & 135 & $73.93$ \\
\bottomrule
\end{tabular*}
\end{table}

\paragraph{Aggregation and interpretation.}
For each observation period, the reported adoption mean is
\begin{equation}
\overline{a}
=
\frac{\sum_m n_m\overline{a}_m}{\sum_m n_m},
\label{eq:production-adoption-aggregation}
\end{equation}
where $\overline{a}_m$ is the monthly mean daily rate. The post-deployment mean is $73.93\%$, computed from the displayed monthly rates. The pre-deployment mean is $23.92\%$, giving a difference of $50.01$ percentage points. The partial month contributes its 13 observed days. The period means average daily amount-based rates.

Higher adoption indicates a larger share of executed transfer amount following accepted algorithmic recommendations, not realized financial benefit. Operating conditions and human behavior can also change across periods, so this observational comparison does not identify a causal deployment effect; the economic estimate comes from the historical backtest.

\section{Maze Navigation}
\label{app:maze-evaluation}

\subsection{Environment and Eligibility}
\label{app:maze-environment}

\paragraph{Layout, information, and deadlines.}
Both layouts use fixed corridors with cross-links and rooms at corridor endpoints. At reset, we sample the occupied quadrants and a random room ordering within each. The first room in the ordering contains the chest, and the prefix of the specified length gives the clue at each stage. Thus clues are nested and include the true room. Within each occupied quadrant, the observation gives equal probability to the unvisited clue candidates until the chest is found. A separate random permutation assigns public room deadlines. At each boundary a new clue is revealed and a rockfall permanently closes the rooms whose deadlines have passed. Figure~\ref{fig:maze-stages} illustrates this progression, and Table~\ref{tab:maze-environment} gives the parameters.

\begin{figure}[htbp]
\centering
\begin{subfigure}[t]{0.60\linewidth}
\centering
\includegraphics[width=\linewidth]{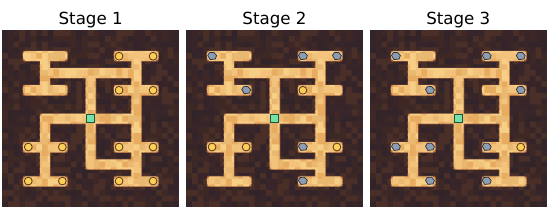}
\caption{maze-16}
\end{subfigure}
\par\smallskip
\begin{subfigure}[t]{\linewidth}
\centering
\includegraphics[width=\linewidth]{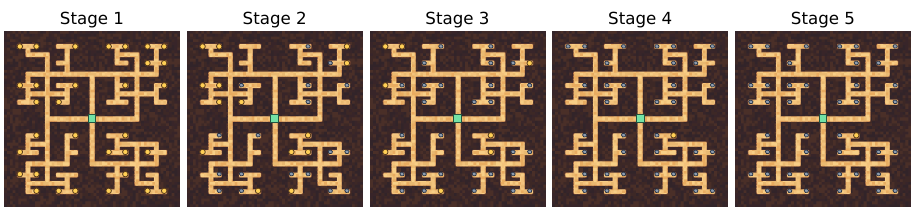}
\caption{maze-48}
\end{subfigure}
\caption{Illustrative clue and rockfall sequences for one instance per layout. Each row shows successive END actions from the start, with no movement. Gold dots mark open treasure candidates, gray stones mark closed rooms, and green squares mark the start. Each illustrated instance has at least one chest accessible through the final stage.}
\label{fig:maze-stages}
\end{figure}

\begin{table}[htbp]
\centering
\small
\caption{Maze environment parameters. Clue sizes count rooms per occupied quadrant. Scheduled cutoffs use the absolute transition clock.}
\label{tab:maze-environment}
\begin{tabular*}{\linewidth}{@{\extracolsep{\fill}}lll@{}}
\toprule
Parameter & maze-16 & maze-48 \\
\midrule
Grid and room count & $31\times31$, 16 rooms & $63\times63$, 48 rooms \\
Chests & One in each of three sampled quadrants & One in each of four quadrants \\
Clue sizes & $(4,2,1)$ & $(8,6,4,2,1)$ \\
Rooms closed at each boundary & $(6,4)$ & $(18,10,8,6)$ \\
Scheduled cutoffs & $(88,124)$ & $(240,304,368,432)$ \\
Episode horizon & $160$ & $496$ \\
Base reward per chest & $12$ & $56$ \\
First-stage reward premium & $25\%$ & $25\%$ \\
Cost per move & $0.05$ & $0.05$ \\
\bottomrule
\end{tabular*}
\end{table}

\paragraph{Actions and rewards.}
A room action moves the agent one tile along a shortest path to the target, which it can change at the next decision. Entering a chest room collects its reward once. END uses one transition, gives zero immediate reward, and advances the stage without movement. The scheduled cutoffs continue to use the same clock after an early END. An episode ends on collecting all chests, being trapped in a newly closed room, exhausting all open treasure candidates, or reaching the horizon.

\paragraph{Observations and eligibility.}
The policy input encodes position, stage, remaining time, fraction of chests remaining, and per-room clue probability, expiry, visitation, and distance. Both interfaces use a common categorical head with one output per room plus END. Flat enables every open, unvisited room. At a nonfinal stage, MMDP keeps only open, unvisited rooms with positive clue probability that expire at the current boundary and can be reached by that boundary, provided at least one such target exists. Otherwise MMDP uses Flat room eligibility. Both allow END at every nonfinal stage. This fixed rule uses observable reachability and deadlines and is evaluated as a heuristic because a move toward a persistent room can improve later reachability.

\subsection{Training and Evaluation}

We train masked categorical PPO from scratch with two hidden layers of 64 units, rollout length $1{,}024$, minibatches of $256$, four optimization epochs per rollout, learning rate $3\times10^{-4}$, and entropy coefficient $0.01$. Flat and MMDP use matched initial weights within each of ten paired training seeds (42--51). Table~\ref{tab:maze-training} gives the discounts, training budgets, and evaluation intervals for the configurations in Figure~\ref{fig:maze-learning}.

\begin{table}[htbp]
\centering
\small
\caption{Maze training and validation settings. Both interfaces use the same settings within each run.}
\label{tab:maze-training}
\begin{tabular*}{\linewidth}{@{\extracolsep{\fill}}lrrr@{}}
\toprule
Maze & Discount $\gamma$ & Training transitions & Validation interval \\
\midrule
maze-16 & $0.990$ & $524{,}288$ & $16{,}384$ \\
maze-48 & $0.999$ & $1{,}048{,}576$ & $32{,}768$ \\
\bottomrule
\end{tabular*}
\end{table}

Validation uses a fixed bank of 128 instances, including evaluations before training and at the final budget. We evaluate final checkpoints greedily on 1,024 independent instances shared by all policies in each layout. Both evaluations report undiscounted returns. Test contrasts resample the ten paired training seeds to obtain marginal 95\% bootstrap intervals. The layouts and evaluation configurations differ in geometry, deadlines, reward scale, discount, and budget. Each paired contrast estimates the interface effect within its setting.

\subsection{Results}

Table~\ref{tab:maze-test} reports held-out returns and paired interface differences for the configurations in Figure~\ref{fig:maze-learning}.

\begin{table}[htbp]
\centering
\small
\caption{Maze test performance. The policy columns report the mean and sample SD across ten training-seed means.}
\label{tab:maze-test}
\begin{tabular*}{\linewidth}{@{\extracolsep{\fill}}lrrrc@{}}
\toprule
Maze & $\gamma$ & Flat (SD) & MMDP (SD) & MMDP $-$ Flat [95\% CI] \\
\midrule
maze-16 & $0.990$ & $14.21\ (2.06)$ & $16.85\ (0.94)$ & $+2.64$ [$1.13,4.00$] \\
maze-48 & $0.999$ & $17.32\ (7.88)$ & $54.95\ (13.69)$ & $+37.64$ [$30.51,43.40$] \\
\bottomrule
\end{tabular*}
\end{table}

\paragraph{Exploration before the first stage change.}
We replay the final greedy policies and record movement, distinct rooms entered, and chests collected before the first stage ends or the episode terminates. Early END means advancing before the first scheduled cutoff. In maze-16, the policies use similar numbers of moves, and Flat visits more rooms. In maze-48, Flat advances earlier to reveal updated clues, enters fewer rooms, and collects fewer first-stage chests.

\begin{table}[htbp]
\centering
\small
\setlength{\tabcolsep}{4pt}
\caption{First-stage behavior on the shared test instances. Each entry is Flat / MMDP. Counts are mean per episode, averaged across the ten training seeds.}
\label{tab:maze-stage-one}
\begin{tabular*}{\linewidth}{@{\extracolsep{\fill}}lrrrrrr@{}}
\toprule
Maze & $\gamma$ & Moves & Rooms entered & Chests & Early END (\%) & First action END (\%) \\
\midrule
maze-16 & $0.990$ & $86.4/87.9$ & $5.50/3.58$ & $1.03/0.97$ & $20.0/0.0$ & $0.00/0.00$ \\
maze-48 & $0.999$ & $130.0/229.4$ & $2.70/4.22$ & $0.23/0.56$ & $70.0/8.8$ & $20.00/0.19$ \\
\bottomrule
\end{tabular*}
\end{table}

\end{document}